\documentclass{article} % For LaTeX2e
\usepackage{iclr2024_conference,times}

\usepackage[utf8]{inputenc} % allow utf-8 input
\usepackage[T1]{fontenc}    % use 8-bit T1 fonts
\usepackage{url}            % simple URL typesetting
\usepackage{booktabs}       % professional-quality tables
\usepackage{amsfonts}       % blackboard math symbols
\usepackage{nicefrac}       % compact symbols for 1/2, etc.
\usepackage{microtype}      % microtypography
\usepackage{titletoc}

\usepackage{subcaption}
\usepackage{graphicx}
\usepackage{amsmath}
\usepackage{multirow}
\usepackage{color}
\usepackage{colortbl}
\usepackage{algorithm}
\usepackage{algorithmicx}
\usepackage{algpseudocode}

\usepackage{relsize}

\usepackage{amssymb,amsmath,amsthm}

\newtheorem{theorem}{Theorem}

\newtheorem{corollary}{Corollary}
\newtheorem{proposition}{Proposition}

\usepackage{diagbox}

\usepackage{caption}
\usepackage{subcaption}

\def\ie{\emph{i.e.}}

\usepackage[pdfusetitle]{hyperref}
\usepackage{xcolor}
\hypersetup{
    colorlinks,
    linkcolor={red!50!black},
    citecolor={blue!50!black},
    urlcolor={blue!80!black}
}
\usepackage{url}
\usepackage{cleveref}
\usepackage{courier}
\usepackage{listings}
\usepackage{booktabs}
\usepackage{xcolor}

\usepackage{inconsolata}
\usepackage{url}

\usepackage{footnote}
\usepackage{footmisc}
\makeatother
\makeatletter
\def\url@leostyle{%
  \@ifundefined{selectfont}{\def\UrlFont{\sf}}{\def\UrlFont{\small\ttfamily}}}
\makeatother
\usepackage[shortlabels]{enumitem}

\graphicspath{{../}} % To reference your generated figures, see below.

\title{Dynamics of Adversarial Attacks on Large Language Model-Based Search Engines}

\author{Xiyang Hu\\
Arizona State University\\
\texttt{xiyanghu@asu.edu}
}

\begin{document}

\maketitle

 \begin{abstract}
 %llms are increasing integrated into search engines.
 %valunerable to attacks. which rank higher of personal data
 %we study the dynamics of  
 The increasing integration of Large Language Model (LLM) based search engines has transformed the landscape of information retrieval. However, these systems are vulnerable to adversarial attacks, especially ranking manipulation attacks, where attackers craft webpage content to manipulate the LLM's ranking and promote specific content, gaining an unfair advantage over competitors. In this paper, we study the dynamics of ranking manipulation attacks. We frame this problem as an Infinitely Repeated Prisoners' Dilemma, where multiple players strategically decide whether to cooperate or attack. We analyze the conditions under which cooperation can be sustained, identifying key factors such as attack costs, discount rates, attack success rates, and trigger strategies that influence player behavior. %We analyze the strategic interactions between players, when to attack and when to cooperate. We investigate conditions for cooperation sustainability. %, extending our analysis to scenarios involving multiple players
 %We identify tipping points in system behavior and the impact of factors such as attack cost, discount rate, attack success rate, and trigger strategy. %, and punishment strategies on system dynamics
 %. We find that cooperation is sustainable when the player is forward-looking. 
 We identify tipping points in the system dynamics, demonstrating that cooperation is more likely to be sustained when players are forward-looking. However, from a defense perspective, we find that simply reducing attack success probabilities can, paradoxically, incentivize attacks under certain conditions. Furthermore, defensive measures to cap the upper bound of attack success rates may prove futile in some scenarios. 
 % Our findings have significant implications for LLM system security design%, suggesting the need for adaptive security measures, careful ecosystem structuring, and nuanced regulatory approaches. This work contributes to AI security 
 % by providing a theoretical foundation for understanding and mitigating its vulnerabilities %in next-generation information retrieval systems
 % .
 These insights highlight the complexity of securing LLM-based systems. Our work provides a theoretical foundation and practical insights for understanding and mitigating their vulnerabilities, while emphasizing the importance of adaptive security strategies and thoughtful ecosystem design.
 %These insights underscore the complexity of designing robust security mechanisms for LLM-based systems. Our work provides a theoretical foundation for understanding and mitigating vulnerabilities in LLM-based systems, highlighting the need for adaptive security strategies, thoughtful ecosystem design, and nuanced regulatory approaches.
 \end{abstract}

\section{Introduction}
\label{sec:intro}

Large Language Models (LLMs) have revolutionized natural language processing, with widespread applications in search engines and information retrieval systems. Examples include ChatGPT search, Perplexity AI, Google Search, and Microsoft Bing, where LLMs significantly enhance the interaction of search results through their advanced language understanding and generation capabilities. 

LLM-based search engines typically operate using the Retrieval-Augmented Generation (RAG) framework \citep{fan2024survey}. In this framework, the system retrieves top-related documents from a knowledge base, database, or the web based on their relevance to the user’s query. These retrieved documents are combined with the query into an augmented input, forming a coherent prompt that provides a comprehensive context for the LLM’s response generation. This integrated prompt ensures that the LLM aligns effectively with the user’s intent as expressed in the query while incorporating relevant external information retrieved from knowledge sources.

However, the integration of LLMs into these search engines introduces new vulnerabilities, particularly ranking manipulation attacks \citep{tang2025stealthrankllmrankingmanipulation, nestaas2024adversarial}. These attacks exploit the sensitivity of LLMs to input variations by embedding crafted instructions or misleading content within documents or webpages, enabling attackers to influence LLMs into favoring their content or product over competing ones \citep{pfrommer2024ranking, aggarwal2024geo}.

Securing LLM-based search engines against ranking manipulation attacks presents a multifaceted challenge that differs significantly from traditional search engine optimization (SEO). Traditional SEO operates within a relatively transparent framework, where the ranking criteria—such as keywords, meta tags, and backlinks—are explicitly defined. Manipulative practices, such as keyword stuffing, are generally easier to detect and address \citep{sharma2019brief}. In contrast, ranking manipulation in LLM-based systems exploits the models’ advanced contextual understanding and generative capabilities. These attacks subtly tailor content to align with the LLM’s internal biases or training patterns \citep{zhao-etal-2023-prompt, wallace-etal-2019-universal, hu-etal-2024-evaluating}, steering outputs toward favoring specific perspectives or products. Unlike traditional SEO, these manipulations do not rely on overt tactics, making them more difficult to detect and counteract.

A more critical distinction lies in how these attacks influence the system. In traditional search engines, manipulative techniques primarily impact the similarity score between a user query and an individual document during the retrieval phase, only impacting the ranking of a single document. However, in LLM-based search engines, manipulated documents are incorporated into the prompt input alongside other retrieved documents. These inputs are not processed in isolation; instead, they interact within the context provided to the LLM. Consequently, manipulations in one document can influence how the LLM interprets and prioritizes other documents in the same prompt, amplifying the attack’s impact. This contextual interaction creates a cascading effect, where one manipulated document can distort the perceived relevance or importance of the entire set of retrieved documents.

Moreover, the dynamic and competitive nature of LLM-based search further complicates the development of effective security mechanisms. Multiple content providers often coexist within these systems, introducing a multiplayer dynamic where the actions of one attacker influence—and are influenced by—other actors in the system. This creates a competitive and adaptive environment, where attacks are not isolated events but interconnected strategies. As a result, designing defenses requires more than isolated algorithmic fixes; it necessitates an investigation of underlying market mechanisms and strategic interactions among providers. Effective security mechanisms must anticipate and account for these dynamics, including the continuous feedback loops that arise as providers adapt their strategies in response to one another.
As LLMs increasingly mediate access to information, understanding the nuances of such attacks is crucial for safeguarding the integrity and trustworthiness of LLM-based search engines.

In this paper, we propose a game-theoretic model to analyze the strategic interactions between attackers in LLM-based search engines. We frame the problem as an Infinitely Repeated Prisoners' Dilemma (IRPD), focusing on two competing players who must decide whether to cooperate (refrain from ranking manipulation attacks) or defect (launch ranking manipulation attacks) in each round of the game. Cooperation corresponds to abstaining from ranking manipulative practices, leading to an equitable distribution of market share, while defection represents attempts to distort rankings to gain a competitive edge. 

Our model captures the unique characteristics of adversarial interactions in LLM-based search engines. It incorporates the stochastic success probabilities of attack actions, reflecting the inherent randomness in how LLMs process and rank content. Unlike deterministic systems, LLMs operate on probabilistic principles \citep{bengio2003neural}, meaning their outputs are not deterministic. This probabilistic nature makes manipulative actions not always succeed, as demonstrated by research on prompt sensitivity and adversarial examples \citep{perez2022red, wallace-etal-2019-universal}. This uncertainty directly influences attackers' strategies and risks, which we represent through attack success rates (ASRs) to align with the unpredictable dynamics of real-world adversarial interactions. Additionally, our model also accounts for the costs of executing attacks, where crafting and deploying adversarial content require escalating investments based on the sophistication of the strategy. These trade-offs, between potential gains and resource allocation, mirror the economic constraints attackers face in real-world scenarios. Furthermore, our framework accounts for market degradation caused by simultaneous successful manipulations. When multiple attacks succeed in influencing the same LLM-mediated ranking decision, the output may contain conflicting or low-quality manipulation signals, reducing user trust and shrinking the effective market size \citep{nestaas2024adversarial}. By incorporating these unique dimensions, our model provides a rigorous framework to analyze the strategic dynamics of these adversarial behaviors.

Based on this model, we examine when cooperation can be sustained among players and how various factors influence long-term cooperation. We start with a foundational scenario involving two identical content providers engaged in infinitely repeated interactions, where each player decides whether to cooperate or defect in each round. We derive boundary conditions under which cooperation can be sustained, incorporating key factors such as attack costs, attack success rates, and the discounting of future profits. We then extend the analysis to explore alternative trigger strategies, and examine cases with heterogeneous players who possess differing characteristics. Finally, we analyze multi-player scenarios, including those where a significant number of players defect. These investigations reveal the factors and mechanisms that drive or hinder cooperation and provide actionable insights for mitigating adversarial behaviors in competitive LLM-based search environments.

Our analyses suggest that ranking manipulation attacks present unique challenges driven by the interplay of stochastic attack success rates, cost structures, future discount rates, and market degradation effects. We find that cooperation among content providers is more likely when attack costs are high or when players are sufficiently forward-looking. Interestingly, the relationship between attack success rates and cooperation sustainability is non-monotonic: intermediate attack success rates yield higher gains from attacks compared to low success rates, while also incurring lower degradation risks and costs compared to high success rates; under certain conditions, they can discourage cooperation the most by striking a balance between gain and loss. Moreover, we identify "futile defense regions," where defense measures aimed at capping attack success rates fail to meaningfully reduce defection incentives. %In these regions, the relative payoff dynamics remain unchanged, allowing attackers to continue benefiting from defection even under constrained success probabilities. 
This highlights the need for comprehensive defensive strategies that not only target attack success probabilities but also address broader incentive structures. Additionally, market degradation caused by widespread defection significantly reduces system utility, creating stronger incentives for cooperation in environments with heightened sensitivity to output quality.

%We also find that multi-player dynamics introduce additional complexities, with the increasing number of attackers diminishing individual gains from defection, thereby widening the scope for cooperative outcomes. Players with asymmetric attack costs or success probabilities exhibit distinct strategic behaviors; for instance, those with lower costs or higher success rates face stronger temptations to defect, requiring more stringent conditions to sustain cooperation. These findings inform the design of effective defense mechanisms that combine technical safeguards—such as dynamic cost adjustments, degradation penalties, and market-based deterrents—with economic incentives to mitigate adversarial attack risks.

Our contributions make both significant theoretical advancements and provide actionable practical insights. We propose a game-theoretic framework specifically tailored to analyze the adversarial dynamics in LLM-based search engines, capturing the unique characteristics of this challenge. On the theoretical side, we derive conditions for sustaining cooperation under various scenarios, offering a rigorous foundation for understanding adversarial interactions. On the practical side, we provide actionable strategies to mitigate these adversarial behaviors effectively. Moreover, we identify influential participants among heterogeneous players whose characteristics critically impact system stability, enabling the design of targeted interventions to enhance resilience. Our findings extend beyond search engines, with broad applicability to LLM-driven applications such as recommendation systems and other platforms susceptible to adversarial threats. By equipping system designers with practical insights, this work contributes to building more secure and trustworthy LLM-based ecosystems.

 \section{Related Literature}
 \label{sec:related}

Our work draws upon and connects several streams of research: (1) vulnerabilities and security challenges in LLM, (2) ranking manipulation in LLM-based search engines, (3) applications of game theory to security problems, and (4) strategic interactions in AI-driven markets. We discuss each of these areas and position our contributions within the existing literature.

\subsection{Vulnerabilities in Large Language Models}

The vulnerability of LLMs to adversarial inputs forms the technical foundation for ranking manipulation attacks in LLM-based search engines. Early work by \cite{wallace-etal-2019-universal} demonstrated universal adversarial triggers could dramatically alter language models' outputs, establishing the basis for exploiting these models' sensitivity to carefully crafted inputs. This sensitivity was further explored by \cite{zhu2023promptbench}, who developed a comprehensive framework for evaluating LLM robustness against various types of adversarial prompts. \cite{perez2022red} conducted red teaming of language models at scale, revealing systematic vulnerabilities across different attack types.

The most prevalent method of executing an attack on large language models is through prompt injection. Prompt injection attacks represent a substantial threat, involving carefully crafted malicious inputs designed to subtly but effectively manipulate the model’s output \citep{Yi2023BenchmarkingAD}. Building on this idea, \cite{Greshake2023NotWY} investigated the domain of indirect prompt manipulation, demonstrating how seemingly harmless inputs can be exploited to influence LLM behavior in unexpected ways. 

Recent research has advanced increasingly sophisticated attack techniques. \cite{jones2023automatically} and \cite{wen2024hard} demonstrated how gradient-informed optimization techniques can generate adversarial inputs that consistently bypass LLM safety measures. The transferability of these attacks across different models, as shown by \cite{zou2023universal}, highlights the systemic nature of these vulnerabilities. Additionally, \cite{wang-etal-2024-reinforcement-learning} introduced reinforcement learning approaches for generating targeted adversarial prompts, while \cite{yi-etal-2024-vulnerability} exposed how reverse alignment techniques can undermine LLM safety mechanisms. \cite{NEURIPS2024_2f148634} formulates jailbreaking as a reinforcement learning-guided search problem. %, significantly outperforming prior black-box methods through a tailored DRL system with transferability across multiple LLMs. 
\cite{NEURIPS2024_70702e8c} demonstrates a tree-based highly effective black-box jailbreaking method that automatically refines and prunes prompts.

\subsection{Ranking Manipulation in LLM-enhanced Search}

The integration of LLMs into search engines has introduced novel attack surfaces for ranking manipulation. \cite{nestaas2024adversarial} provided one of the first comprehensive analyses of how adversaries can use carefully crafted external content, such as website text or plugin documentation, to manipulate LLMs into promoting specific preferences or products, raising concerns about the security and reliability of LLM-driven search engines. \cite{pfrommer2024ranking} extended this work by identifying specific content structures and semantic patterns that effectively influence LLM-based ranking systems.

Further advancements into efficient attack techniques were proposed by \cite{aggarwal2024geo} and \cite{tang2025stealthrankllmrankingmanipulation}, who developed optimization-based techniques for manipulating LLM preferences in search contexts. These works demonstrate that ranking manipulation in LLM-based systems differs fundamentally from traditional SEO, requiring more sophisticated approaches that exploit the models' deep language understanding and generation capabilities.

While these studies establish the technical feasibility of ranking manipulation attacks, they primarily focus on one-time individual attack instances. Our work complements this literature by examining the long-term strategic dynamics when multiple attackers repeatedly interact within the market.

\subsection{Game Theory in Security Applications}

Our game-theoretic approach builds upon established frameworks for analyzing security challenges. \cite{manshaei2013game} provided a comprehensive foundation for applying game theory to security problems, demonstrating how game theory modeling can illuminate attacker-defender dynamics. \cite{alpcan2010network} provided foundational work on applying game theory to network security. \cite{kamhoua2021game} extended these game theory principles along with machine leanring to cyber-security systems, offering insights relevant to our analysis of LLM-based search engines.

Recent works have examined specific security scenarios using game theory. Particularly relevant is work by \cite{10.1145/3337772} on defensive deception strategies, which demonstrates how game theory can model complex security interactions where attackers and defenders must reason about each other's strategies. \cite{Roy2010ASO} further showed how game-theoretic models can capture the strategic nature of security investments and defense mechanisms.  \cite{NEURIPS2020_0ea6f098} presented a game-theoretic analysis of additive adversarial attacks and defenses, introducing novel geometry-flavored proof techniques to analyze provable attacks and defenses.

Our work extends these approaches to the novel context of LLM-based search engines, incorporating unique elements such as stochastic attack success rates, LLM degradation under universal attacks, and the interconnected nature of multiple attackers' actions.

%\cite{10.1145/3337772} provided a taxonomic survey of defensive deception strategies in cybersecurity, offering a framework for understanding attacker-defender dynamics.

%Building on this, \cite{do2017game} explore the use of game theory in modeling and analyzing the interactions between attackers and defenders in cyber-physical systems, offering insights that are relevant to our analysis of LLM-based search engines.

\subsection{Strategic Interactions in AI-driven Markets}

The competitive dynamics we model align closely with broader strategic interactions observed in AI-driven markets. For instance, \cite{miklos2019collusion} explored how improved demand forecasting from AI affects collusion sustainability. While better forecasts allow firms to tailor prices more precisely, they also increase incentives to undercut rivals during high-demand periods. Similarly, \cite{calvano2020artificial} explored pricing algorithms using Q-learning and showed that such algorithms can independently develop collusive strategies, sustaining supracompetitive prices without direct communication. Together, these studies reveal how AI disrupts traditional competitive paradigms, fostering novel forms of coordination in market environments.

Another area of interest lies in cooperation among self-interested agents. \cite{bi2023understanding} analyzed partnership formation in federated learning environments, demonstrating how to prompt sustainable collaboration in repeated interactions. Complementing this, \cite{banchio2023adaptive} proposed a theoretical framework for collusion between adaptive learning algorithms, demonstrating how spontaneous coupling can lead to profitable coordination beyond static Nash equilibria.

One more relevant research direction is on humans' strategic reaction to AI adoption is crucial. \cite{wang2023algorithmic} examined the impact of algorithmic transparency on firm and user surplus in markets with strategic users, showing that while transparency can enhance a firm’s predictive power and profitability, it does not always benefit users. This highlights the complex trade-offs involved in designing transparent AI systems that balance the interests of all stakeholders. Similarly, \cite{dai2020conspicuous} focused on diagnostic decision-making by experts, identifying scenarios where high-type experts might avoid necessary testing to signal their expertise. This behavior presents barriers to the adoption of AI tools in contexts where human expertise and machine intelligence must coexist. These studies provide valuable insights into the strategic challenges faced when integrating AI into markets reliant on both expertise and technology.

In this paper, we study a novel and increasingly important AI-driven market: LLM-based search engines. This setting presents unique characteristics that differentiate it from previously studied markets. Here, content providers compete for visibility and traffic through a system mediated by large language models, where success depends not just on traditional competitive factors but also on the ability to influence LLM-based content interpretation and ranking. The stochastic nature of LLM outputs, coupled with the interdependent effects of multiple providers' actions and the potential for market-wide degradation from excessive manipulation, creates a distinctive competitive landscape that hasn't been examined in prior work. By modeling these dynamics through an infinitely repeated game framework, we provide the first theoretical analysis of strategic behavior in this emerging market structure.

\section{Model Setup}
\label{sec:model}

We model the problem of ranking manipulation attacks on LLM-based search engines as an infinitely repeated game, a common approach for studying scenarios where players interact over time and learn from each other’s actions \citep{abreu1988theory, dal2018determinants}. In each time period ($t=1, 2, 3, \cdots$), two content providers ($i=1, 2$) simultaneously decide whether to launch a ranking manipulation attack. The problem takes the structure of an Infinitely Repeated Prisoners' Dilemma, where individuals face incentives to deviate from cooperation. In our context, the short-term incentive for each player is to launch an attack to gain a temporary advantage over the competitors. 
%In this context, both players are incentivized to launch attacks, which, if unchecked, would degrade the overall quality of the LLM's outputs, reducing the value of the search engine for users and the potential market for all players involved.
Players discount future profits by a common discount rate $\delta \in (0, 1)$, reflecting a preference for immediate gains over long-term returns. %The use of discounting is standard in repeated game models to capture the time preference of players \citep{fudenberg1991game, maskin2001markov}. This assumption is particularly relevant in digital marketplaces, where businesses often prioritize strategies that deliver short-term advantages due to uncertainties in future market conditions and competition \citep{stucke2013internet}. A higher discount rate indicates a greater emphasis on current-period profits, which can make it harder to sustain cooperation, as the future benefits of refraining from attacks become less valuable.

The ranking manipulation attack against large language models (LLMs) is characterized by an attack success rate (ASR), denoted as $p$, which represents the probability that the attack will successfully achieve its goal—altering the model's response to elevate the rank of a targeted product. This concept is a central measure in LLM safety research, where ASR is commonly used to evaluate the effectiveness of various adversarial attacks \citep{shayegani2023survey}. The value of $p$ is influenced by several factors, including the sophistication of the attack, the strength of the model's defenses, and the resources invested by the attacker. Generally, more sophisticated attacks and greater resources lead to a higher $p$, while robust defenses and countermeasures implemented by the LLM or search engine reduce the probability of success.

When launching a preference manipulation attack, the player incurs a cost, denoted as $c$. This cost represents the resources needed to develop, deploy, and maintain the manipulative content, such as research and development effort, data acquisition, and server resources. We examine different cost function forms, including constant cost, linear cost ($c \propto p$), and quadratic cost ($c \propto p^k, k>1$). The latter one reflects the increasing difficulty and diminishing return of developing more effective attacks%, a characteristic observed in adversarial examples in machine learning \citep{}
. This cost structure suggests that attackers must balance the potential gains from successful attacks against the resources required to implement them. %More sophisticated attacks that increase the likelihood of success require higher investments, as captured by a functional form like $c(p) = e^{\alpha p} - 1$, where $p$ is the probability of a successful attack, and $\alpha$ is a scaling parameter. 

For the market, in each period, there is one unit mass of potential consumers, with total demand normalized to 1. The goods supplied by the providers are perfect substitutes. This assumption simplifies the analysis while capturing the competition for market share between the players. It allows us to focus on how the division of market share is influenced by the players' strategic choices. %Similar demand structures have been used in models of competitive markets and online advertising \citep{athey2013efficient, agrawal2019contextual}. The assumption of a single unit of demand also reflects the typical scenario where multiple players vie for the attention of users in a digital ecosystem.

Each player can observe the actions of the other at the end of each period, including whether they choose to cooperate or attack. This forms a perfect monitoring game, a standard assumption in repeated game theory, which enables players to make strategic adjustments based on the observed behaviors of others \citep{fudenberg1991game}. This perfect-monitoring assumption is a baseline simplification. In deployed LLM-based search systems, content providers usually observe public signals, such as ranking positions, citations, referral traffic, and visible changes in generated answers, rather than competitors' private actions directly. We use perfect monitoring for two reasons. First, it isolates the incentive channel created by $p$, $c$, $\beta$, and $\delta$ from the separate statistical problem of inferring whether an attack occurred. Second, it gives a best-case monitoring environment for sustaining cooperation: if cooperation is difficult even when deviations are observed, noisy ranking signals would only make discipline harder. An imperfect-monitoring extension would replace observed actions with public signals over ranking outcomes and would require stronger continuation incentives.
%In practice, digital platforms often have visibility into the activities of competitors, such as tracking changes in content or monitoring shifts in ranking. This perfect monitoring enables strategies like the "trigger strategy," where one party retaliates against any observed defection, leading to sustained cycles of attacks if trust is broken \citep{abreu1988theory}. Such strategies are relevant in competitive environments where trust can quickly erode, leading to negative outcomes for all participants.

\subsection{Payoff Structure} 

The strategic interactions between the players are represented by a payoff matrix, structured as a Prisoners' Dilemma, as summarized in Table~\ref{table:prisoners_dilemma}. Within each cell of the table, the first element denotes the payoff of player 1, and the second element denotes that of player 2. The payoff values depend on whether each player chooses to cooperate or launch an attack:

\begin{table}[h!]
\centering
\begin{tabular}{c|cc} 
\hline
\diagbox{\textbf{\textbf{Player 1}}}{\textbf{Player 2}} & \textbf{Cooperate} & \textbf{Attack}  \\ 
\hline
\textbf{Cooperate}                                      & R, R               & S, T             \\ 
\hline
\textbf{Attack}                                         & T, S               & Q, Q             \\
\hline
\end{tabular}
\caption{Payoff Matrix}
\label{table:prisoners_dilemma}
\end{table}

\begin{itemize} \item $R = \frac{1}{2}$ (\textbf{Mutual Cooperation}): If both players refrain from launching attacks, they equally share the market demand, resulting in a payoff of $\frac{1}{2}$ each. %This reflects a stable state where neither party gains an advantage through manipulation, aligning with outcomes in cooperative game theory \citep{maskin2001markov}.

\item $T = p + (1-p)\frac{1}{2} - c$ (\textbf{Temptation Payoff}): When one player launches an attack while the other cooperates, the attacker potentially captures the entire market if the attack is successful with probability $p$. If the attack fails (with probability $1-p$), the market demand is split evenly, but the attacker still bears the cost $c$ of launching the attack. %This situation mirrors asymmetric outcomes in advertising competitions, where aggressive strategies can yield disproportionate market gains \citep{varian2007position}.

\item $S = (1-p)\frac{1}{2}$ (\textbf{Sucker Payoff}): If a player cooperates while the other attacks, the cooperator retains some market share only if the attack fails. Otherwise, the cooperator loses all market share. %, reflecting the vulnerability of non-manipulative actors in competitive digital markets \citep{}.

\item $Q = p^2 \frac{1}{2}\beta + p(1-p) + (1-p)^2 \frac{1}{2} - c$ (\textbf{Mutual Attack}): If both players launch attacks, the outcomes depend on the success rate of the attacks. (1) If both attacks succeed (probability $p^2$), they equally share the market but at a degraded value due to reduced LLM output quality, represented by $\beta < 1$. A smaller $\beta$ indicates a larger degradation. (2) If only one player succeeds in attacking (probability $p(1-p)$), that player monopolizes the market. (3) If both attacks fail, the market is split evenly. The condition $\beta < 1$, which represents the degraded output quality when both parties launch attacks, aligns with empirical findings showing that when all parties engage in ranking manipulation attacks, it results in detrimental outcomes for everyone involved \citep{nestaas2024adversarial}. %reflects how adversarial behavior can degrade user experience, reducing the value of the market, consistent with observations in adversarial impacts on recommendation systems \citep{nestaas2024adversarial}.
\end{itemize}

The condition $T > R > Q > S$ preserves the structure of a Prisoner's Dilemma, indicating that while mutual cooperation is preferable, individual incentives drive the players to defect. This condition holds when 
$c < \frac{1}{2}p + \frac{1}{2}(\beta - 1)p^2$, ensuring that the temptation to attack outweighs the costs but results in a worse outcome for both when both attack.

\section{Analysis}
\label{sec:analysis}

In this section, we analyze the Infinitely Repeated Prisoners' Dilemma (IRPD) model to examine the conditions under which cooperation can be sustained among players. We answer key questions regarding what factors influence the sustainability of cooperation, and how varying parameters—such as attack costs, attack success rates, and future profits discount rates—affect the decision to either cooperate or defect. This analysis provides insights into the mechanisms that can encourage cooperation and mitigate ranking manipulation attacks in LLM-driven information retrieval systems.

%\subsection{Grim Trigger Strategy}

We consider the grim trigger strategy, a classic trigger strategy in repeated games where players cooperate until one defects; after a defection, all players respond by always defecting. This strategy serves as a baseline for understanding how cooperation might be enforced in an environment where deviations are possible. It creates a strong deterrent for initial defection since any deviation leads to permanent mutual defection, which can degrade outcomes for all parties involved. In Section~\ref{sec:tft}, we extend our analysis to alternative trigger strategies.

We first derive the discounted payoffs for continuous cooperation $V(C)$ and for a one-time defection followed by mutual defection $V(D)$.
If both players always cooperate, the discounted payoff for each player is:
$$
V(C) = R + \delta R + \delta^2 R + \dots = \frac{R}{1-\delta}
$$
If the first defecting player defects once while the other cooperates, and then both players defect forever after, the payoff for the first defecting player is:
$$
V(D) = T + \delta Q + \delta^2 Q + \dots = T + \frac{\delta Q}{1-\delta}
$$
Here, $R$, $T$, and $Q$ represent the payoffs for mutual cooperation, defection, and mutual defection, respectively, while $\delta$ represents the discount rate that measures how players value future profits relative to immediate gains. $V(D)$ is a combination of the immediate gain from defection and the discounted future payoffs under mutual defection.

\subsection{Condition for Cooperation}

For cooperation to be a rational strategy for each player, the payoff from continuous cooperation must be at least as high as the payoff from defection followed by mutual defection $V(C) \ge V(D)$. Theorem~\ref{theo:coop} provides a summary of the result.

\begin{theorem}[Cooperation Condition]
Two players prefer long-term cooperation over engaging in ranking manipulation attacks if and only if:
\[ \delta \geq \delta^* = \frac{T - R}{T - Q} = \frac{p - 2c}{p - \beta p^2 + p^2} \]
where $\delta^*$ is the critical discount factor.
\label{theo:coop}
\end{theorem}

The discount factor reflects how much players value future profits relative to immediate gains.  Higher values of $\delta$ imply that players are more forward-looking and are thus more willing to cooperate because attacking now would mean losing significant future profits. Conversely, a lower $\delta$ makes attacking more appealing, as players prioritize immediate rewards gained from ranking attacks over long-term benefits. Cooperation is only viable if $\delta \geq \delta^*$, emphasizing the need for creating environments where long-term outcomes are valued.

\begin{proof}
For cooperation to be sustainable, we need $V(C) \geq V(D)$, which is
$$\frac{R}{1-\delta} \geq T + \frac{\delta Q}{1-\delta}
$$

Reorganizing this inequality gives the critical condition for sustaining cooperation:
$$\delta \geq \frac{T - R}{T - Q} \overset{\underset{\mathrm{def}}{}}{=} \delta^*$$
where $\delta^*$ is the critical discount factor. This threshold represents the minimum value of $\delta$ required for players to prioritize long-term cooperation over the short-term gains from defection.

Let's expand this condition using our given payoff functions:
\begin{align*}
T - R &= \left[p + (1-p)\left(\frac{1}{2}\right) - c\right] - \frac{1}{2} = \frac{p}{2} - c \\
T - Q &= \left(\frac{1 + p}{2} - c\right) - \left(p^2\frac{1}{2}\beta + p(1-p) + (1-p)^2\frac{1}{2} - c\right) \\
    &= \frac{1 + p}{2} - p^2\frac{1}{2}\beta - p(1-p) - \frac{(1-p)^2}{2} \\
    &= \frac{1 + p}{2} - p^2\frac{1}{2}\beta - p + p^2 - \frac{1}{2} + p - \frac{p^2}{2} \\
    &= \frac{p}{2} - p^2\frac{1}{2}\beta + \frac{p^2}{2}
\end{align*}

Therefore, the players prefer cooperation over launching an attack when:
$$\delta \ge \delta^* = \frac{\frac{p}{2} - c}{\frac{p}{2} - p^2\frac{1}{2}\beta + \frac{p^2}{2}} = \frac{p - 2c}{p - \beta p^2 + p^2}
$$
\end{proof}

\begin{corollary}
    The cooperation will be sustained if and only if the cost is larger than a threshold:
    $$c \ge \frac{p - \delta(p - \beta p^2 + p^2)}{2}$$
    \label{coro:coop}
\end{corollary}
Corollary~\ref{coro:coop} suggests that a higher attack cost reduces the attractiveness of attacking, as the immediate gain from a successful attack is offset by a substantial cost to launch the attack. 

In Theorem~\ref{theo:delta-star}, we analyze how different parameters affect the range of the cooperation-inducing $\delta$ and discuss their broader implications for cooperative behavior in the system.

\begin{theorem}[Monotonicity of $\delta^*$]
The critical discount factor $\delta^*$ exhibits the following behavior:
\begin{itemize}
    \item $\delta^*$ decreases as the attack cost $c$ increases. This implies that higher attack costs make cooperation easier to maintain, as the relative benefit of defecting diminishes.
    
    \item $\delta^*$ increases with larger $\beta$. When $\beta$ is high, the payoffs from mutual defection are relatively high, making defection more attractive and cooperation more difficult to sustain.
    
    \item $\delta^*$ is non-monotonic with respect to the attack success rate $p$. This means that the relationship between $p$ and $\delta^*$ is not straightforward—an increase in $p$ may either raise or lower the likelihood of sustaining cooperation, depending on the interplay of other parameters like $c$ and $\beta$.
\end{itemize}
\label{theo:delta-star}
\end{theorem}

\begin{proof}
We derive these results by analyzing the partial derivatives of $\delta^*$ with respect to the key parameters $c$, $\beta$, and $p$. For $c$, the partial derivative is negative, indicating that an increase in attack costs reduces the threshold for cooperation. For $\beta$, the derivative is positive, signifying that higher mutual defection payoffs make cooperation harder to maintain. The derivative with respect to $p$ is more complex, as it can be positive or negative depending on the values of the other parameters, leading to the observed non-monotonic behavior.
\end{proof}

We find that the cost of attack $c$ plays a critical role in fostering cooperation by lowering the critical discount factor $\delta^*$. An increase in $c$ discourages players from defecting, as the expense of launching an attack outweighs its short-term benefits. This insight highlights the importance of measures that raise the cost of malicious actions, such as enhancing LLM security technologies or implementing stricter legal penalties. By making defection more costly, these interventions shift incentives in favor of long-term cooperative behavior, creating a more stable ecosystem.

The degradation factor $\beta$ also significantly affects the sustainability of cooperation. It reflects how much the market value diminishes when mutual defection occurs, which degrade the performance of LLMs. As $\beta$ increases, the payoffs from mutual defection (when both players attack) become more attractive, which undermines the incentive to cooperate. Lowering $\beta$ effectively widens the range of cooperation inducing $\delta$, making cooperation more likely to be sustained. In practical terms, this suggests that systems should be designed to heavily penalize mutual defection. For instance, designing LLM-driven search engines where mutual defection leads to severe degradation in output quality can reduce the benefits of defection, thereby promoting cooperation. 

The attack success rate $p$ introduces complexity to the dynamics of cooperation due to its non-monotonic relationship with $\delta^*$. It affects both the temptation to defect and the risks associated with defection. When $p$ is low, attacks are less likely to succeed, which may disincentivize players from defecting. However, as $p$ increases, the temptation to attack rises because a successful attack yields significant short-term gains. Interestingly, beyond a certain threshold, further increases in $p$ can counter-intuitively reduce the temptation to defect, because the mutual defection that follows successful attacks degrades the market significantly and the cost of developing a stronger attack becomes higher, both reducing the attractiveness of defection. This nuanced interplay indicates that managing $p$ requires careful calibration, as its effects on cooperation depend on its interactions with other parameters such as $c$ and $\beta$. We will investigate how $p$ impacts cooperation in depth in the following sections.

These findings underscore the broader implications of parameter manipulation for designing cooperative systems. Raising the cost of attacks, lowering the payoffs from mutual defection, and carefully managing the attack success probability can together create an environment conducive to sustained cooperation. In the following sections, we delve deeper into how these parameters interact and explore practical strategies for implementing these insights in real-world LLM-driven ecosystems.

\subsection{Cooperation Formation Region}
\label{sec:coop-form}

Due to the lack of a closed-form solution for the cooperation formation condition in terms of $p$, we employ numerical visualization to investigate the parameter space where cooperation can be sustained. 

Figure~\ref{fig:cooperation_region} visualizes the regions of the \(\delta\) (discount factor) and \(p\) (attack success probability) space that results in long-term cooperation under various values of \(\beta\) and different cost functions. The region to the right of the curve (the blue region) is where $\delta > \frac{p - 2c}{p - \beta p^2 + p^2}$, \ie, the cooperation formation region. We can derive several key observations from these plots.

First, across all cost functions, the region where cooperation is possible tends to shrink as \(\beta\) increases. This is evident as the blue regions, where the inequality condition is satisfied, become smaller moving from the first row (where \(\beta = 0.2\)) to the fourth row (where \(\beta = 0.8\)). The decrease in the size of the cooperation region with increasing \(\beta\) can be attributed to the increasing sensitivity of the players to defection payoffs. As \(\beta\) increases, the loss of degradation decreases, therefore the incentive for defection grows, thereby requiring higher values of \(\delta\) to sustain cooperation.
\begin{figure}[t]
\centering
\includegraphics[width=\linewidth]{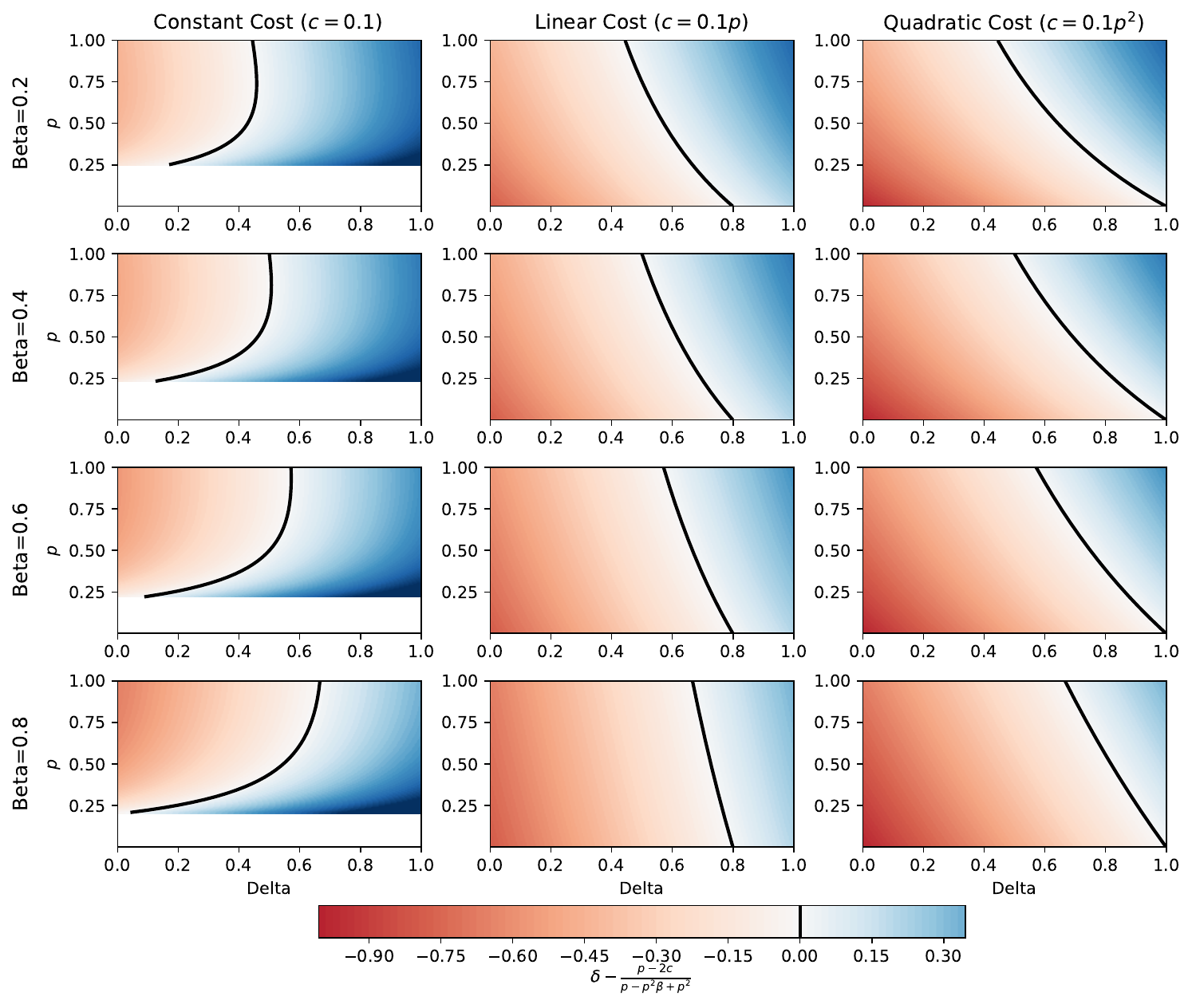}
\vspace{-0.8cm}
\caption{Region of Corporation Formation (the region to the right of the boundary) %{\color{red} TODO: exponential cost func}
}
%\vspace{0.5cm}
\label{fig:cooperation_region}
\end{figure}
Second, the size and shape of the cooperation region vary significantly across different cost functions. For the constant cost function ($c=0.1$), the cooperation region is relatively larger across different values of $\delta$. In contrast, with a linear cost function ($c=0.1p$), the cooperation region becomes smaller, as the cost is overall smaller than the linear setting. Finally, for the quadratic cost function ($c=0.1p^2$), the cooperation region is the smallest, as the cost values are even smaller compared to the linear and constant cases. This minimal cost significantly diminishes the likelihood of maintaining cooperation. While the magnitude of the cost primarily determines the overall size of the cooperation region, the form of the cost function governs the shape of the cooperation boundary. The curvature and slope of the boundary reflect the sensitivity of cooperation to changes in $\delta$, $\beta$ and $p$.

Third, there exists a lower bound for \(\delta\), below which cooperation is not feasible. This lower bound represents the minimum level of patience (or preference for future rewards) required for players to consider cooperating. Interestingly, this lower bound for \(\delta\) may increase or decrease as \(p\) increases, depending on the specific cost function \(c\) and value of \(\beta\). In some scenarios, for example, the subfigures in the second and third columns, a higher success probability \(p\) makes it easier to maintain cooperation, thus lowering the minimum required \(\delta\). In other cases, like the first column subfigures, increasing \(p\) might raise the lower bound for \(\delta\), making cooperation harder to achieve unless players are sufficiently future-oriented. In addition, for some settings where $\beta$ is small, such as the top left subfigure, although increasing \(p\) disincentive cooperation at the beginning, after a certain point, $p$ further can unexpectedly decrease the motivation to defect, as the increasing success of attacks leads to a larger loss of mutual defection, which substantially harms the market and diminishes the appeal of defection.

In summary, the ability to sustain long-term cooperation depends critically on the interplay between the discount factor \(\delta\), success probability \(p\), and the form of the cost function \(c\). Cooperation is more likely when the discount factor is high, and the cost is high. As the cost decreases, the conditions for cooperation become more stringent, shrinking the feasible region for long-term cooperation. These findings suggest that both the design of cost structures and the understanding of players' time discount sensitivities are crucial for fostering stable cooperative relationships.

\subsection{Payoff Analysis of Cooperation and Defection in LLM Systems}
\label{sec:payoff-analysis}

%These plots reveal a range of dynamics in the context of Large Language Model (LLM) security, where players decide between cooperating or defecting based on their expected payoffs.

In this section, we examine the payoff values for cooperation $V_C$ and defection $V_D$, focusing on how varying factors such as attack success probabilities $p$ %, attack costs ($c$), and the discount rate $\delta$ 
affect strategic decisions. Figure~\ref{fig:payoff_beta} illustrates the payoffs for cooperation $V_C$ and defection $V_D$ as $p$ varies. In Figure~\ref{fig:payoff_beta}, we fix $\beta$ at $0.4$ because the pattern of the cooperation region is relatively consistent across $\beta$ values, as we see in Figure~\ref{fig:cooperation_region}. Cooperation is sustainable when the $V_C$ curve lies above the $V_D$ curve, indicating that the long-term benefits of maintaining cooperative behavior outweigh the immediate gains from defection.

%\begin{figure}[t]
\begin{figure}[t]
\includegraphics[width=\linewidth,
  trim=2mm 2mm 2mm 0mm,
  clip]{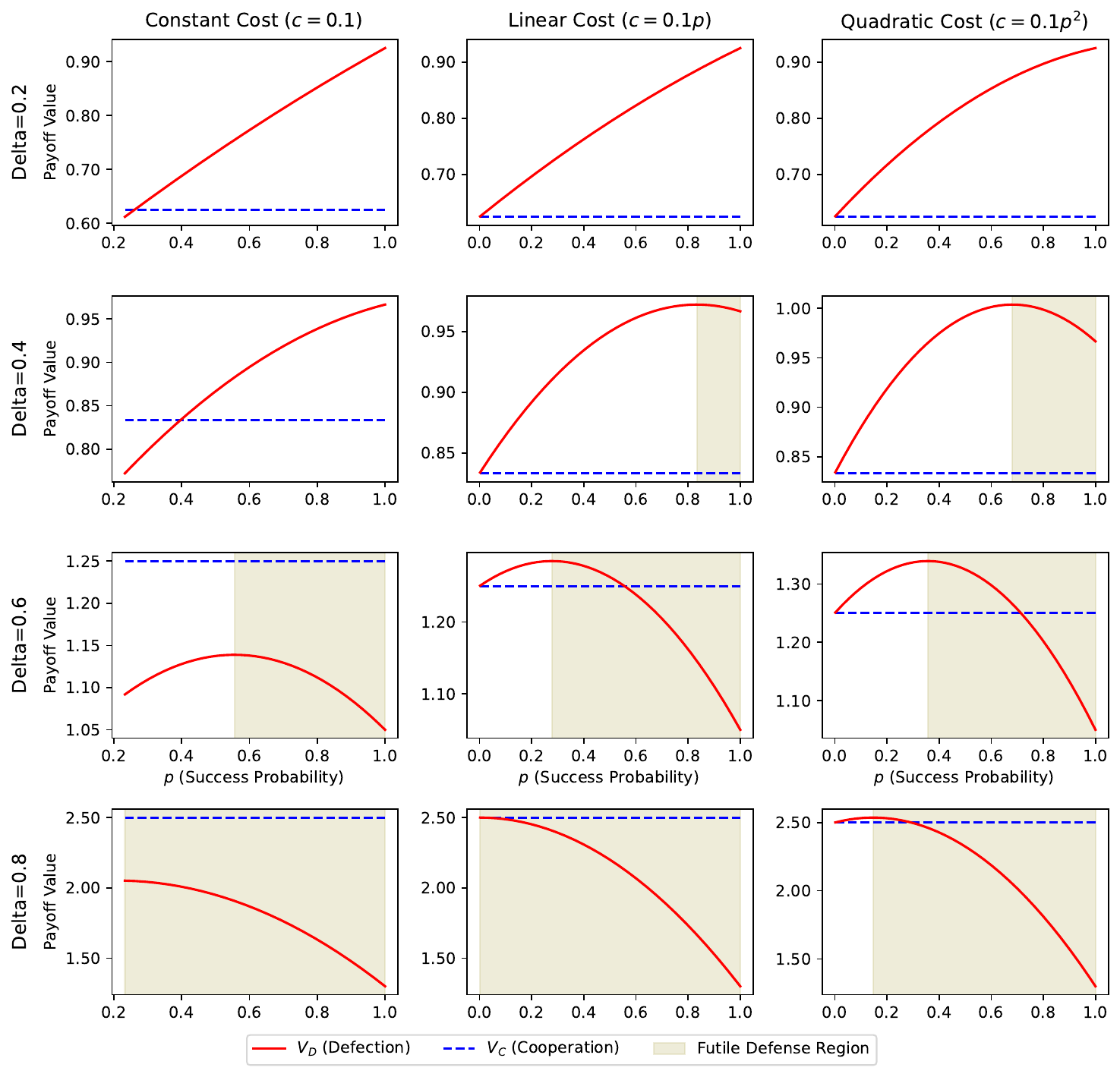}
\caption{$V_C$ and $V_D$ values ($\beta = 0.4$)}
\label{fig:payoff_beta}
\end{figure}
%\end{figure}

Our analysis uncovers several non-intuitive findings that have significant implications for the security design of LLM-based systems. These insights challenge conventional assumptions about how to maintain cooperation and deter adversarial behavior.

% One of the most unexpected findings from our analysis is that reducing the attack success probability, $p$, can sometimes incentivize attacks. This observation contradicts the intuitive belief that lowering $p$ discourages attacks by reducing the chance of a successful outcome.

\begin{proposition}[Defection payoff and cap-based defenses]
Fix $\delta\in(0,1)$, $\beta\in(0,1)$, and let $c(p)$ be a differentiable attack cost. Under Grim Trigger, the one-shot defection payoff followed by mutual defection is
\[
V_D(p)
=
\frac{1}{2}+\frac{p}{2}-c(p)
+
\frac{\delta}{1-\delta}
\left(
\frac{1}{2}
-\frac{(1-\beta)p^2}{2}
-c(p)
\right).
\]
The payoff gap between defection and cooperation is
\[
V_D(p)-V_C
=
\frac{(1-\delta)p-\delta(1-\beta)p^2-2c(p)}
{2(1-\delta)}.
\]
If $c''(p)\ge 0$, then $V_D(p)$ is concave in $p$. Any interior maximizer $p^\star$ is the unique solution to
\[
c'(p^\star)+\delta(1-\beta)p^\star
=
\frac{1-\delta}{2}.
\]
If no solution lies in $(0,1)$, the maximum is attained at an endpoint.

For the cost functions used in the figures:
\[
\begin{aligned}
c(p)=c_0
&\quad\Rightarrow\quad
p^\star=\frac{1-\delta}{2\delta(1-\beta)},\\
c(p)=\alpha p
&\quad\Rightarrow\quad
p^\star=\frac{(1-\delta)/2-\alpha}{\delta(1-\beta)},\\
c(p)=\alpha p^2
&\quad\Rightarrow\quad
p^\star=\frac{1-\delta}{2(2\alpha+\delta(1-\beta))},
\end{aligned}
\]
whenever the displayed value belongs to $(0,1)$.
\label{prop:cap}
\end{proposition}

\begin{proof}
Using the payoff definitions,
\[
T=\frac{1}{2}+\frac{p}{2}-c(p)
\]
and
\[
Q
=
p^2\frac{\beta}{2}
+p(1-p)
+(1-p)^2\frac{1}{2}
-c(p).
\]
The mutual-defection payoff simplifies to
\[
Q
=
\frac{1}{2}
-\frac{(1-\beta)p^2}{2}
-c(p).
\]
Therefore,
\[
V_D(p)
=
T+\frac{\delta Q}{1-\delta}
=
\frac{1}{2}+\frac{p}{2}-c(p)
+
\frac{\delta}{1-\delta}
\left(
\frac{1}{2}
-\frac{(1-\beta)p^2}{2}
-c(p)
\right).
\]
Since
\[
V_C=\frac{1}{2(1-\delta)},
\]
we obtain
\[
V_D(p)-V_C
=
\frac{(1-\delta)p-\delta(1-\beta)p^2-2c(p)}
{2(1-\delta)}.
\]

Differentiating $V_D(p)$ gives
\[
V_D'(p)
=
\frac{1}{2}
-
\frac{c'(p)+\delta(1-\beta)p}{1-\delta}.
\]
Thus any interior maximizer satisfies
\[
c'(p^\star)+\delta(1-\beta)p^\star
=
\frac{1-\delta}{2}.
\]
The second derivative is
\[
V_D''(p)
=
-\frac{c''(p)+\delta(1-\beta)}{1-\delta}.
\]
If $c''(p)\ge0$, then $V_D''(p)<0$, so the maximizer is unique whenever it is interior. The closed-form cases follow by substituting $c'(p)=0$, $c'(p)=\alpha$, and $c'(p)=2\alpha p$, respectively.
\end{proof}

\begin{corollary}[Non-monotone effect of reducing $p$]
When an interior maximizer $p^\star$ exists, reducing $p$ can increase the incentive to defect. In particular, for $p>p^\star$, a local decrease in $p$ raises $V_D(p)$ because $V_D'(p)<0$ on that side of the peak.
\end{corollary}

\begin{corollary}[Low-success attacks can remain attractive]
If $c(0)=0$ and $c'(0)<(1-\delta)/2$, then there exists $\varepsilon>0$ such that
\[
V_D(p)>V_C
\qquad
\text{for all } p\in(0,\varepsilon).
\]
Thus, driving $p$ close to zero does not by itself guarantee strict cooperation when low-success attacks have near-zero marginal cost. If $c(0)>0$, this conclusion need not hold.
\end{corollary}

\begin{corollary}[Futile defense region]
Suppose an attacker can choose any $p\in[0,\bar p]$ after a defense imposes an upper bound $\bar p$ on attainable attack success. If $V_D(p^\star)>V_C$ and $\bar p\ge p^\star$, then
\[
\max_{p\in[0,\bar p]} V_D(p)
=
V_D(p^\star)
=
\max_{p\in[0,1]} V_D(p).
\]
Hence, the cap does not reduce the maximum attainable defection payoff. A cap-based defense can change the attacker's best-response payoff only if it lowers the attainable upper bound below $p^\star$, and it restores cooperation only if
\[
\max_{p\in[0,\bar p]}V_D(p)\le V_C.
\]
\end{corollary}

The findings from our analysis emphasize the necessity of a system-level perspective when designing defenses for LLM-based platforms. Instead of solely focusing on reducing the attack success probability $p$, effective strategies must account for the interplay between $p$, the market degradation factor $\beta$, the discount rate $\delta$, and the attack cost $c$. These parameters collectively shape the strategic decisions of both attackers and defenders. For example, dynamic defense mechanisms that adjust the LLM’s response strategies based on detected attack patterns could mitigate persistent threats by introducing variability that increases the cost of launching repeated attacks. Moreover, integrating measures that penalize attackers—whether or not their attacks succeed—can shift the cost-benefit analysis against continued adversarial attempts. This could involve mechanisms that reduce the utility of even low-probability attacks, thereby discouraging persistent efforts to exploit the system. By balancing these considerations, LLM developers can design more resilient systems capable of maintaining stability in the face of evolving adversarial tactics.

\subsection{Empirical Calibration of $p$ and $c$}

The model treats $p$ as a reduced-form probability that an attack changes the LLM-mediated ranking enough to obtain the attacker-favored allocation. Existing ranking-manipulation papers do not always report this exact probability, but their experimental quantities give useful calibration ranges.

\citet{nestaas2024adversarial} report that direct preference-manipulation attacks can succeed in roughly $95\%$--$100\%$ of trials in some search settings, while external-page attacks have lower success, reaching at most about $25\%$. Their plugin experiments also show selection rates increasing from $0\%$ to above $90\%$ in some cases. \citet{pfrommer2024ranking} report normalized ranking-score improvements from $54.23\%$ to $95.74\%$ across models. \citet{tang2025stealthrankllmrankingmanipulation} report average target ranks between $1.46$ and $2.50$ on $10$-item lists and between $1.87$ and $2.39$ on $8$-item lists. Using the rank-normalized proxy
$p_{\mathrm{rank}}=\frac{K-\bar r}{K-1}$,
where $K$ is the candidate-list size and $\bar r$ is the average target rank, these values correspond roughly to $p_{\mathrm{rank}}\in[0.80,0.95]$.

The cost parameter $c$ is measured in units of one-period normalized market value, so dollar costs must be scaled by the economic value of the relevant query class. \citet{pfrommer2024ranking} report inference costs of approximately $\$15$, $\$50$, and $\$450$ for their experiments across different providers. If the one-period gross profit associated with the relevant query class is $G$, these budgets map to
$c\approx \frac{15}{G},\quad
c\approx \frac{50}{G},\quad
c\approx \frac{450}{G}$.
For example, if $G\in[10^3,10^5]$, this gives a rough range
$c\in[1.5\times 10^{-4},4.5\times 10^{-1}]$.
We use this range only as a sensitivity guide, since deployment costs, query value, and attacker automation differ across markets.

\section{Tit-for-Tat Trigger Strategy}
\label{sec:tft}

Following the structure of our model above, this section investigates the Tit-for-Tat (TFT) strategy as an alternative trigger strategy to foster cooperation. We analyze three different settings for the Tit-for-Tat (TFT) strategy: (1) a single defection by one player followed by immediate one-time retaliation from the other, (2) alternating cooperation and defection between the players, and (3) consecutive defections for a fixed number of rounds before returning to cooperation. For each setting, we examine the payoff of defection \(V(D)\), and compare it with the payoff of continuous cooperation \(V(C)\), to derive conditions under which cooperation can be sustained.

\subsection{Setting 1: Player 1 Defects in the First Round, Player 2 Retaliates Once}

In this setting, we consider a scenario where Player 1 defects in the first round, breaking cooperation, but then chooses to cooperate in all subsequent rounds. Player 2 retaliates by defecting in the second round as a response to Player 1's initial defection, and then resumes cooperation from the third round onward. This setup models a limited retaliation strategy, where defection is punished but not perpetuated indefinitely, aligning with the principles of the Tit-for-Tat strategy. The action sequences for the two players are as follows: Player 1's actions are \(D \to C \to C \to C \to \dots\); while Player 2's actions are \(C \to D \to C \to C \to \dots\).

\begin{theorem}[Cooperation Condition Under Single Defection and One-Time Retaliation]
Long-term cooperation is sustainable under this strategy if and only if:
\[
\delta \geq \frac{T - R}{R - S} = 1 - \frac{2c}{p}
\]
\label{theo:tft-onetime}
\end{theorem}

\begin{proof}
    The payoff for continuous cooperation, denoted \(V(C)\), is straightforward and represents the long-term reward from mutual cooperation across all rounds. It can be expressed as:
    \[
    V(C) = \frac{R}{1 - \delta},
    \]
    where \(R\) is the reward for mutual cooperation, and \(\delta\) is the discount factor that accounts for the value players place on future payoffs. In contrast, the payoff for defection by Player 1, \(V(D)\), reflects the one-time benefit of defection in the first round, followed by the sucker payoff \(S\) in the second round when Player 2 retaliates, and the cooperation payoff \(R\) in all subsequent rounds. The defection payoff is therefore given by:
    \[
    V(D) = T + \delta S + \sum_{t=3}^\infty \delta^{t-1}R = T + \delta S + \frac{\delta^2}{1 - \delta}R,
    \]
    where \(T\) is the temptation payoff received in the first round, \(S\) is the sucker payoff in the second round, and the remaining terms account for the discounted future cooperation payoffs.

    For cooperation to be sustainable, \(V(C) \geq V(D)\):
    \[
    \frac{R}{1 - \delta} \geq T + \delta S + \frac{\delta^2}{1 - \delta}R.
    \]

    Subtract \(\frac{\delta^2}{1 - \delta} R\) from both sides:
    \[
    \frac{R}{1 - \delta} - \frac{\delta^2}{1 - \delta} R \geq T + \delta S.
    \]
    
    Thus, the inequality becomes:
    \[
    (1 + \delta) R \geq T + \delta S.
    \]

    % We know that:
    % \[
    % 1 - \delta^2 = (1 - \delta)(1 + \delta).
    % \]
    % Substitute this into the inequality:
    % \[
    % \frac{(1 - \delta)(1 + \delta) R}{1 - \delta} \geq T + \delta S.
    % \]
    % Cancel the \((1 - \delta)\) term:
    % \[
    % (1 + \delta) R \geq T + \delta S.
    % \]

    % Expanding the left-hand side gives:
    % \[
    % R + \delta R \geq T + \delta S.
    % \]

    We can rearrange the inequality:
    \[
    \delta \geq \frac{T - R}{R - S} = 1 - \frac{2c}{p}.
    \]
\end{proof}

\subsection{Setting 2: Alternating Cooperation and Defection}

In this setting, the two players engage in a cyclic pattern of alternating cooperation and defection. Player 1 begins by defecting in the first round, prompting Player 2 to retaliate by defecting in the second round. This behavior continues indefinitely, with Player 1 defecting in all odd-numbered rounds and cooperating in all even-numbered rounds, while Player 2 cooperates in odd-numbered rounds and defects in even-numbered rounds. The alternating pattern models a competitive dynamic where neither player consistently cooperates nor defects.
The action sequences for the two players can be summarized as follows:  
Player 1’s actions are \(D \to C \to D \to C \to \dots\);  
Player 2’s actions are \(C \to D \to C \to D \to \dots\)

\begin{theorem}[Cooperation Condition Under Alternating Cooperation and Defection]
Long-term cooperation is sustainable under this strategy if and only if:
\[
\delta \geq \frac{T - R}{R - S}=1 - \frac{2c}{p}
\]
\end{theorem}

\begin{proof}
    The payoff for continuous cooperation, \(V(C)\), remains unchanged, representing the long-term reward for mutual cooperation:
    \[
    V(C) = \frac{R}{1 - \delta}.
    \]
    For alternating cooperation and defection, Player 1 receives the temptation payoff \(T\) in odd-numbered rounds and the sucker payoff \(S\) in even-numbered rounds. The total payoff is an infinite geometric series:
    \[
    V(D) = T + \delta S + \delta^2 T + \delta^3 S + \cdots = \sum_{n=0}^\infty \delta^{2n}T + \sum_{n=0}^\infty \delta^{2n+1}S.
    \]
    Simplifying the series yields:
    \[
    V(D) = \frac{T + \delta S}{1 - \delta^2}.
    \]
    
    To sustain cooperation, the payoff from mutual cooperation must exceed the payoff from alternating cooperation and defection:
    \[
    \frac{R}{1 - \delta} \geq \frac{T + \delta S}{1 - \delta^2}.
    \]
    Simplifying this inequality gives the same condition as Setting 1:
    \[
    \delta \geq \frac{T - R}{R - S}=1 - \frac{2c}{p}.
    \]
\end{proof}

\subsection{Setting 3: Consecutive Defections for k Rounds}

In this setting, Player 1 defects for \(k\) consecutive rounds before returning to cooperation. Player 2 retaliates by defecting from the second round through round \(k+1\), matching Player 1’s defections, before resuming cooperation in round \(k+2\). This scenario models a prolonged period of defection followed by reconciliation, testing the players' capacity to return to cooperative behavior after extended conflict.
The action sequences for the two players can be summarized as follows:  Player 1’s action sequence is $\underbrace{D \to D \to \dots \to D }_{k \text{ rounds}}\to C \to C \to \dots$ (defects for \(k\) rounds, then cooperates); Player 2’s action sequence is \(C \to \underbrace{D \to \dots \to D \to D}_{k \text{ rounds}} \to C \to \dots\) (cooperates in the first round, retaliates for \(k\) rounds, then cooperates).

\begin{theorem}
It is rational for Player 1 to either defect for only one round (\(k = 1\)) or defect indefinitely (i.e., \(k \to \infty\)), depending on the value of the discount factor \(\delta\). Specifically, Player 1 will:
\begin{itemize}
    \item Defect for only one round if \(\delta \ge \frac{Q - S}{R - S} = p\beta + (1-p) - \frac{2c}{p}\).
    \item Defect indefinitely if \(\delta < \frac{Q-S}{R-S} = p\beta + (1-p) - \frac{2c}{p}\).
\end{itemize}
\label{theo:defect-one-indef}
\end{theorem}

The implications of Theorem~\ref{theo:defect-one-indef} reveal that Player 1's decision to defect either for only one round (\(k = 1\)) or indefinitely (\(k \to \infty\)) depends on the value of the discount factor \(\delta\). When Player 1 defects only once, this behavior is equivalent to the scenario described in Setting 1, where a single defection is followed by a return to cooperation. Specifically, if \(\delta \geq p\beta + (1-p) - \frac{2c}{p}\), Player 1 defects once and then resumes cooperation, as the long-term value of cooperation outweighs the immediate gains of prolonged defection. However, if \(\delta < p\beta + (1-p) - \frac{2c}{p}\), the immediate rewards from defection dominate, prompting Player 1 to defect indefinitely, equivalent to the Grim Trigger strategy.

Combining this theorem with the Theorem~\ref{theo:tft-onetime} in Setting 1 provides a complete characterization of Player 1’s behavior across the entire range of \(\delta\). If \(\delta \geq 1 - \frac{2c}{p}\), cooperation is sustained indefinitely, as both players value the long-term rewards of mutual cooperation over any short-term temptation. If \(\delta\) lies in the intermediate range, \(p\beta + (1-p) - \frac{2c}{p} \leq \delta < 1 - \frac{2c}{p}\), Player 1 defects for one round before returning to cooperation, as this strategy balances the short-term gain from defection with the long-term benefits of cooperation. Finally, if \(\delta < p\beta + (1-p) - \frac{2c}{p}\), Player 1 defects indefinitely, leading to the breakdown of cooperation.

These combined results emphasize the pivotal role of \(\delta\) in governing strategic behavior. A higher \(\delta\) fosters long-term cooperation by ensuring that the future rewards of mutual cooperation outweigh the short-term incentives to defect. Conversely, a lower \(\delta\) diminishes the value of future payoffs, incentivizing short-sighted strategies such as prolonged or indefinite defection. This underscores the importance of designing systems or environments that increase \(\delta\)—for example, by fostering repeated interactions, implementing reputation systems, or introducing long-term incentives—to encourage cooperative outcomes and mitigate the risk of defection.

\begin{proof}
    The payoff for continuous cooperation, \(V(C)\), remains the same:
    \[
    V(C) = \frac{R}{1 - \delta}.
    \]
    For Player 1, the payoff from defecting for \(k\) rounds includes: Temptation payoff (\(T\)) in the first round, Mutual defection payoff (\(Q\)) in rounds 2 through \(k\), Sucker payoff (\(S\)) in round \(k+1\), and Cooperation payoff (\(R\)) from round \(k+2\) onward.
    
    Thus, the total payoff for Player 1 is:
    \[
    V(D) = T + \sum_{i=1}^{k-1} \delta^i Q + \delta^k S + \sum_{t=k+2}^{\infty} \delta^{t-1} R = T + \left(\frac{\delta - \delta^k}{1 - \delta}\right)Q + \delta^k S +  \frac{\delta^{k+1}}{1 - \delta}R .
    \]
    Here, the summation for rounds 2 to \(k\) is a geometric series, while the cooperation payoff beyond round \(k+1\) accounts for the discounted long-term rewards.

    % For cooperation to be sustainable, the payoff from always cooperating must be at least as high as the payoff from defecting for \(k\) rounds:
    % \[
    % \frac{R}{1 - \delta} \geq T + \left( \frac{\delta - \delta^k}{1 - \delta} \right)Q + \delta^k S + \frac{\delta^{k+1}}{1 - \delta}R.
    % \]
    % Rearrange terms:
    % \[
    % \frac{R}{1 - \delta} - \frac{\delta^{k+1}}{1 - \delta}R \geq T + \left( \frac{\delta - \delta^k}{1 - \delta} \right)Q + \delta^k S.
    % \]
    % Factor out \(\frac{1}{1 - \delta}\) from the left side:
    % \[
    % \frac{(1 - \delta^{k+1}) R}{1 - \delta} \geq T + \left( \frac{\delta - \delta^k}{1 - \delta} \right)Q + \delta^k S.
    % \]
    % Multiply both sides by \(1 - \delta\):
    % \[
    % (1 - \delta^{k+1}) R \geq (1 - \delta)T + (\delta - \delta^k)Q + (1 - \delta)\delta^k S .
    % \]
    
    We now analyze whether it is beneficial for Player 1 to defect for \(k + 1\) rounds instead of \(k\). The change in Player 1’s payoff when increasing defection rounds from \(k\) to \(k + 1\) is: (1) The payoff from mutual defection \(Q\) increases by \(\delta^k Q\); (2) The sucker payoff \(S\) decreases by \((\delta^k - \delta^{k+1}) S\); (3) The cooperation payoff \(R\) decreases by \(\delta^{k+1} R\).
    
    Thus, the net change in payoff is:
    \[
    \delta^k Q - (\delta^k - \delta^{k+1}) S - \delta^{k+1} R.
    \]
    Factor the terms:
    \[
    \delta^k \left[ Q - S \right] - \delta^{k+1} \left[ R - S \right].
    \]
    Defecting for more rounds is beneficial if the above expression is positive:
    \[
    \delta^k \left[ Q - S \right] - \delta^{k+1} \left[ R - S \right] > 0.
    \]
    Simplifying this gets us the condition when the player has an incentive to defect for more rounds:
    \[
    \delta < \frac{Q-S}{R - S}.
    \]
\end{proof}

\subsection{Cooperation Formation Region under Tit-for-Tat}

\begin{figure}[t]
\centering
\includegraphics[width=\linewidth]{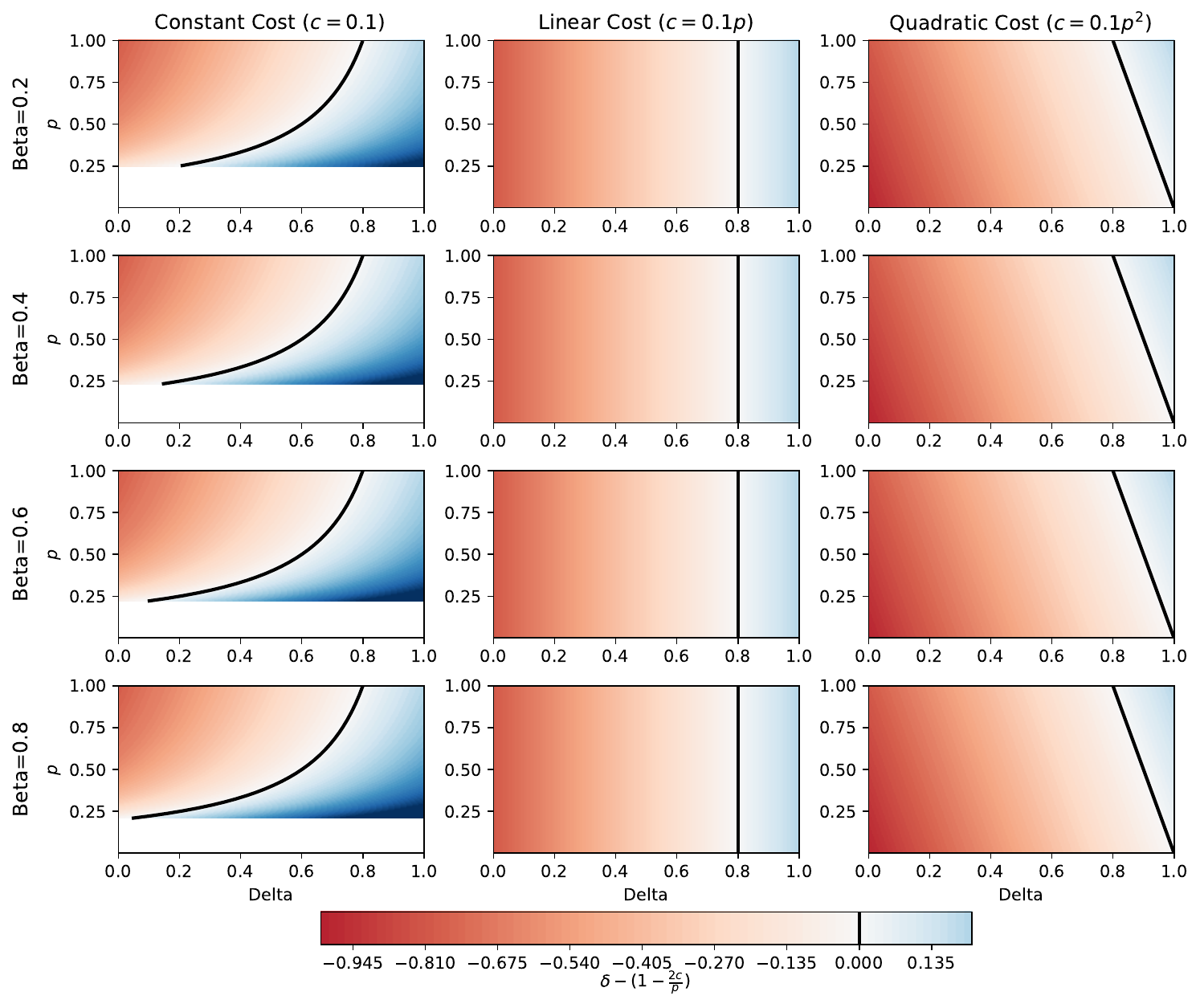}
\vspace{-0.8cm}
\caption{Region of Cooperation Formation (the region to the right of the boundary, Tit-for-Tat)}
%\vspace{0.5cm}
\label{fig:cooperation_region_titfortat}
\end{figure}

Under the Tit-for-Tat strategy, the condition for sustaining cooperation is given by $\delta \geq 1 - \frac{2c}{p}$, where \(\delta\) is the discount factor, \(c\) is the cost of launching an attack, and \(p\) is the attack success probability. This condition is notable for being independent of the degradation factor \(\beta\), setting Tit-for-Tat different from the Grim Trigger strategy, where \(\beta\) plays a central role in shaping the cooperation region under the Grim Trigger strategy. This independence fundamentally influences the dynamics of cooperation under Tit-for-Tat, leading to unique properties and implications for the formation and sustainability of cooperative behavior.

Under Grim Trigger, the cooperation condition depends on \(\beta\), as $\delta \geq \frac{p - 2c}{p - \beta p^2 + p^2}$. This condition explicitly incorporates \(\beta\), allowing Grim Trigger to leverage long-term market degradation to enforce cooperation. When \(\beta\) is low (i.e., mutual defection severely degrades market value), Grim Trigger imposes strong penalties for defection, effectively deterring deviations. However, when \(\beta\) is high, mutual defection outcomes become less severe, diminishing the deterrent effect of Grim Trigger's punishment and shrinking the cooperation region.

In contrast, for Tit-for-Tat, the cooperation condition is the same under different \(\beta\) values.
The lack of dependence on \(\beta\) in Tit-for-Tat arises from its short-term punishment mechanism. Unlike Grim Trigger, which enforces cooperation by leveraging long-term degradation in market value (influenced by \(\beta\)), Tit-for-Tat relies solely on immediate retaliation. As our analysis above shows, the Tit-for-Tat strategy only has two possible settings: one player only defects once and another player only retaliates once in the next round, or two players alternate between cooperation and defection. As a result, the two players will never attack simultaneously so that the LLM systems would never degrade, and the cooperation region remains unaffected by \(\beta\), providing stability across varying degradation levels.

Figure~\ref{fig:cooperation_region_titfortat} visualizes the cooperation formation regions for Tit-for-Tat under various cost functions and attack success probabilities. The regions to the right of the curve (the blue regions) represent the parameter space where cooperation can be sustained. Unlike Grim Trigger, these regions are shaped solely by the interplay between \(\delta\), \(c\), and \(p\), with no direct dependence on \(\beta\).

While this independence from \(\beta\) simplifies the analysis, it also highlights a limitation of Tit-for-Tat: its inability to exploit long-term degradation as a deterrent. In scenarios where \(\beta\) is low (i.e., mutual defection severely degrades market value), Grim Trigger effectively uses this dynamic to enforce cooperation by imposing a strong, permanent penalty. Tit-for-Tat, lacking this mechanism, struggles to sustain cooperation when the immediate incentives to defect are strong.

%In contrast, the independence from \(\beta\) makes Tit-for-Tat stable across varying degradation levels. This stability ensures that the cooperation region remains unaffected by \(\beta\), which can be advantageous in markets with uncertain or variable degradation factors. However, this same independence limits Tit-for-Tat’s ability to exploit long-term dynamics for enforcing cooperation, making it less effective in scenarios where leveraging market degradation is crucial.

%The independence of Tit-for-Tat from \(\beta\) provides stability across varying degradation levels but limits its ability to deter defection through long-term penalties. In contrast, Grim Trigger’s reliance on \(\beta\) enables it to enforce cooperation effectively in scenarios with low degradation factors but makes it more vulnerable when \(\beta\) is high. The simpler dependence of Tit-for-Tat on \(p\) and \(c\) makes it more stable and predictable but less robust in adversarial settings where long-term dynamics are critical. These differences highlight the trade-offs between adaptability and enforcement strength in fostering cooperation in LLM-driven adversarial environments.

%{\color{blue} sanction?}

\section{One-Time Fixed Cost}

In this section, we extend the analysis to account for scenarios where players face a one-time fixed cost when launching an attack. This cost could arise from factors such as one-time hardware or software purchases, one-time research and development expenditures, or other non-recurring investments required to enable the attack. We replicate the cooperation condition analysis and the payoff analysis of cooperation and defection, as done in Sections~\ref{sec:analysis}, but now incorporating the one-time fixed cost, denoted as $c$, incurred upon the first attack by a player.

\begin{theorem}[Cooperation Condition with One-Time Fixed Cost]
In the presence of a one-time fixed cost, two players will prefer long-term cooperation over engaging in ranking manipulation attacks if and only if the following condition is satisfied:
\[ \delta \geq \delta^*_{\text{One-Time Cost}} = \frac{p - 2c}{p - p^2 \beta + p^2 - 2c} \]
where $\delta^*_{\text{One-Time Cost}}$ represents the critical discount factor required to sustain cooperation.
\label{theo:coop-onetime}
\end{theorem}

\begin{proof}
For cooperation to be sustainable, the condition \( V(C) \geq V(D)_{\text{One-Time Cost}} \) must hold, where \( V(D)_{\text{One-Time Cost}} \) is the defection payoff in the presence of a one-time fixed cost. This is given by:  
\begin{align*}  
V(D)_{\text{One-Time Cost}} &= T + \delta Q + \delta^2 Q + \dots \\  
&= T + \frac{\delta Q}{1-\delta} \\  
&= \left(\frac{1}{2} + \frac{1}{2}p - c\right) + \frac{\delta}{1-\delta}\left(p^2 \beta + p(1 - p) + (1 - p)^2 \frac{1}{2}\right),  
\end{align*}  
where:  
\begin{itemize}  
    \item \( T = \frac{1}{2} + \frac{1}{2}p - c \) represents the immediate payoff when one player defects while the other cooperates, accounting for the one-time fixed cost \( c \).  
    \item \( Q \) represents the discounted future payoffs under mutual defection. For subsequent rounds, there are no additional costs incurred.  
\end{itemize}  

By solving the inequality \( V(C) \geq V(D)_{\text{One-Time Cost}} \), we obtain the critical threshold:  
\[
\delta \geq \delta^*_{\text{One-Time Cost}} = \frac{p - 2c}{p - p^2 \beta + p^2 - 2c}.  
\]
\end{proof}

Compared to the case with recurring costs (Theorem~\ref{theo:coop}), the presence of a one-time fixed cost causes the cooperation region to contract. This is because the threshold for sustaining cooperation raises, as shown in the increasing the left-hand side of the inequality.
The critical condition for cooperation established in Theorem~\ref{theo:coop-onetime} highlights the distinct influence of the fixed cost on equilibrium outcomes. Specifically, the presence of the \(-2c\) term in the denominator reflects the one-time nature of the cost, which does not penalize repeated acts of defection. This limitation makes the fixed cost less effective at discouraging defection compared to recurring costs, which impose a continuous penalty over successive interactions. 

Furthermore, the one-time fixed cost is more effective at discouraging attacks from myopic players with smaller $\delta$, who prioritize immediate payoffs, compared to forward-looking players with a long-term outlook who value future rewards through a larger $\delta$. Myopic players view the fixed cost as a significant hurdle because it substantially reduces their immediate net benefits, making attacks less appealing unless the short-term gains are exceptionally high. In contrast, forward-looking players are less deterred by the fixed cost, as they treat it as an upfront investment and consider the long-term benefits of defection. Consequently, while the fixed cost effectively shrinks the attack incentives for myopic players, its impact on forward-looking players is less pronounced. This distinction underscores the need for strategic calibration of fixed costs to deter attacks from a broader range of players, including those with long-term strategies, ensuring overall system stability and cooperation.

Finally, the findings under recurring costs (Theorem~\ref{theo:coop}) remain consistent when applied to the case of a one-time fixed cost. The cooperation formation region in the presence of one-time fixed cost exhibits similar sensitivity to \(\delta\), \(p\), and \(\beta\), where higher values of \(\beta\) and lower values of costs shrink the region for sustained cooperation. In addition, the non-monotonic relationship between \(p\) and defection payoff persists, where increasing \(p\) can initially discourage cooperation but, beyond a certain threshold, may diminish the motivation for defection as mutual defection harms both players significantly. Furthermore, the existence of futile defense regions is also observed under a one-time fixed cost, where reducing \(p\) does not necessarily lower defection payoffs, leaving attackers' incentives largely unaffected. These findings highlight that while a fixed cost shifts the magnitude of cost impacts, the underlying dynamics of the cooperation formation region and its dependencies on the strategic parameters remain consistent with those derived under recurring costs.

\section{Asymmetric Players Scenarios}

In this section, we extend our analysis to scenarios where players differ in their attack success probabilities, attack costs, and discount rates. Such asymmetries are common in real-world competitive environments where participants have varying resources, capabilities, or strategic priorities. We explore how these differences influence the conditions for sustaining cooperation and identify which player's characteristics are pivotal in determining the overall cooperation dynamics.

The sustainability of cooperation in asymmetric scenarios is primarily determined by the player with the greatest temptation to defect. This player has the most significant influence on whether mutual cooperation can be maintained.

We consider three types of asymmetries: differences in attack success probabilities, differences in attack costs, and differences in discount rates. In the first scenario, the players have different probabilities of successfully executing an attack, with Player 1 having a lower success probability than Player 2 (\(p_1 < p_2\)). In the second scenario, the players face different costs for conducting an attack, with Player 1 incurring a lower cost than Player 2 (\(c_1 < c_2\)). In the third scenario, the players have differing levels of patience, reflected in their discount factors, where Player 1 values future payoffs less than Player 2 (\(\delta_1 < \delta_2\)). For each of these asymmetries, the payoff structures are modified to reflect the differences between the players. We then derive the conditions under which cooperation can be sustained.

\subsection{Scenario 1: Different Attack Success Probabilities ($p_1 < p_2$)}

When players have different probabilities of successfully launching an attack, their incentives to cooperate or defect diverge, as reflected in their respective payoff structures. This asymmetry arises in settings where one player has a greater likelihood of achieving successful attacks, creating unequal temptations to defect.

To account for the asymmetric success probabilities (\(p_1 < p_2\)), the payoff functions are adjusted as follows:

\begin{itemize}
    \item \textbf{Cooperation Payoff:}
    \[
    R_1 = R_2 = \frac{1}{2}.
    \]
    
    \item \textbf{Temptation Payoffs:}
    \begin{align*}
    T_1 &= p_1 + (1 - p_1) \left( \frac{1}{2} \right) - c, \\
    T_2 &= p_2 + (1 - p_2) \left( \frac{1}{2} \right) - c.
    \end{align*}

    \item \textbf{Sucker Payoffs:}
    \begin{align*}
    S_1 &= (1 - p_2) \left( \frac{1}{2} \right), \\
    S_2 &= (1 - p_1) \left( \frac{1}{2} \right).
    \end{align*}
    
    \item \textbf{Mutual Defection Payoff:}
    \begin{align*}
    Q_1 &= p_1 p_2 \frac{1}{2}\beta + p_1 (1 - p_2) + (1 - p_1)(1 - p_2) \left( \frac{1}{2} \right) - c, \\
    Q_2 &= p_1 p_2 \frac{1}{2}\beta + p_2 (1 - p_1) + (1 - p_1)(1 - p_2) \left( \frac{1}{2} \right) - c.
    \end{align*}
    
\end{itemize}

\paragraph{Cooperation Condition}

Cooperation is sustainable if:

\begin{align*}
\delta_i &\geq \frac{T_i - R_i}{T_i - Q_i}, \quad \text{for } i = 1, 2 \tag{Grim Trigger}
\\
\delta_i &\geq \frac{T_i - R_i}{R_i - S_i}, \quad \text{for } i = 1, 2 \tag{Tit-For-Tat}
\end{align*}

\begin{theorem}
\label{thm:asym_p}
Consider two players with attack success probabilities \(p_1\) and \(p_2\), cooperation is sustainable if and only if:
\begin{align*}
\delta_1 \geq \frac{p_1 - 2c}{(1 - \beta)p_1 p_2 + p_2} \quad \text{and} \quad \delta_2 \geq  \frac{p_2 - 2c}{(1 - \beta)p_1 p_2 + p_1} \tag{Grim Trigger}
\end{align*}
\begin{align*}
\delta_1 \geq \frac{p_1 - 2c}{p_2} \quad \text{and} \quad \delta_2 \geq \frac{p_2 - 2c}{p_1} \tag{Tit-For-Tat}
\end{align*}
\end{theorem}

\begin{proof}
For each player, the discounted payoff from continuous cooperation must be at least as high as the payoff from defecting once and then facing mutual defection indefinitely.
\end{proof}

\begin{corollary}
Under different attack success probabilities ($p1 < p2$), the sustainability of cooperation is primarily determined by the player with the higher success probability. %If \(p_i\) is higher for one player, that player has a greater temptation to defect, and
\end{corollary}

\begin{proof}
Under the grim trigger strategy, given $p1 < p2$, we have $\frac{p_1 - 2c}{(1 - \beta)p_1 p_2 + p_2} < \frac{p_2 - 2c}{(1 - \beta)p_1 p_2 + p_1}$ for Grim Trigger, and $\frac{p_1 - 2c}{p_2} < \frac{p_2 - 2c}{p_1}$ for Tit-For-Tat. Under both strategies, player 2 has a greater temptation to defect and a stricter cooperation condition.
\end{proof}

Players with higher attack success probabilities face greater temptation to defect, as their potential short-term gains from defection are larger. Thus, the sustainability of cooperation hinges on their willingness to prioritize future payoffs over immediate gains.

\subsection{Scenario 2: Different Attack Costs (\(c_1 < c_2\))}

In scenarios where players incur different costs for launching attacks, the player with the lower cost faces reduced deterrence against defecting. This asymmetry shifts the balance of cooperation dynamics.

The payoff functions reflect the cost asymmetry:

\begin{itemize}
    \item \textbf{Cooperation Payoff:}
    \[
    R_1 = R_2 = \frac{1}{2}.
    \]
    
    \item \textbf{Temptation Payoffs:}
    \begin{align*}
    T_1 &= p + (1 - p) \left( \frac{1}{2} \right) - c_1, \\
    T_2 &= p + (1 - p) \left( \frac{1}{2} \right) - c_2.
    \end{align*}

    \item \textbf{Sucker Payoffs:}
    \begin{align*}
    S_1 &= (1 - p) \left( \frac{1}{2} \right), \\
    S_2 &= (1 - p) \left( \frac{1}{2} \right).
    \end{align*}
    
    \item \textbf{Mutual Defection Payoff:}
    \begin{align*}
    Q_1 = p^2 \frac{1}{2}\beta + p(1 - p) + (1 - p)^2 \left( \frac{1}{2} \right) - c_1, \\
    Q_2 = p^2 \frac{1}{2}\beta + p(1 - p) + (1 - p)^2 \left( \frac{1}{2} \right) - c_2.
    \end{align*}
\end{itemize}

\begin{theorem}
\label{thm:asym_c}
Let two players have different attack costs \(c_1\) and \(c_2\). The cooperation is sustainable if and only if:
\begin{align*}
\delta_1 \geq \frac{p - 2c_1}{(1 - \beta)p^2 + p} \quad \text{and} \quad \delta_2 \geq  \frac{p - 2c_2}{(1 - \beta)p^2 + p} \tag{Grim Trigger}
\end{align*}
\begin{align*}
\delta_1 \geq \frac{p - 2c_1}{p} \quad \text{and} \quad \delta_2 \geq \frac{p - 2c_2}{p} \tag{Tit-For-Tat}
\end{align*}
\end{theorem}

\begin{proof}
The proof follows similar logic to Theorem~\ref{thm:asym_p}, adjusting for the different costs in the temptation and mutual defection payoffs.
\end{proof}

\begin{corollary}
Under different attack costs ($c_1 < c_2$), cooperation sustainability is determined by the player who has a lower attack cost.
\end{corollary}

\begin{proof}
Under different attack cost, given $c_1 < c_2$, we have $\frac{p - 2c_1}{(1 - \beta)p^2 + p} > \frac{p - 2c_2}{(1 - \beta)p^2 + p}$ for Grim Trigger, and $\frac{p - 2c_1}{p} > \frac{p - 2c_2}{p}$ for Tit-For-Tat. Under both strategies, player 1 has a greater temptation to defect and a stricter cooperation condition.
\end{proof}

A player with a lower attack cost faces less deterrent against attacking. They set the cooperation threshold, as their defection can disrupt mutual cooperation.

\subsection{Scenario 3: Different Discount Rates (\(\delta_1 < \delta_2\))}

Discount rate asymmetries reflect that players value future profits differently, with the less patient player placing less value on future cooperation benefits. This influences their willingness to cooperate.

\begin{theorem}
\label{thm:asym_delta}
For players with different discount rates \(\delta_1\) and \(\delta_2\), cooperation is sustainable if and only if:
\begin{align*}
c \ge \frac{p - \delta_1(p - \beta p^2 + p^2)}{2} \quad \text{and} \quad c \ge \frac{p - \delta_2(p - \beta p^2 + p^2)}{2}  \tag{Grim Trigger}
\end{align*}
\begin{align*}
c \ge \frac{1}{2}(1-\delta_1)p \quad \text{and} \quad c \ge \frac{1}{2}(1-\delta_2)p  \tag{Tit-For-Tat}
\end{align*}
\end{theorem}

\begin{proof}
The proof follows a similar logic to Theorem~\ref{thm:asym_p}, adjusting for the different discount rates.
\end{proof}

\begin{corollary}
The player with a lower discount rate is more inclined to defect, making their cooperation threshold critical for sustaining cooperation.
\end{corollary}

\begin{proof}
    Under different discount rates, given $\delta_1 < \delta_2$, we have $\frac{p - \delta_1(p - \beta p^2 + p^2)}{2} > \frac{p - \delta_2(p - \beta p^2 + p^2)}{2}$ for Grim Trigger, and $\frac{1}{2}(1-\delta_1)p > \frac{1}{2}(1-\delta_2)p$ for Tit-For-Tat. Under both strategies, player 1 has a greater temptation to defect and a stricter cooperation condition.
\end{proof}

\subsection{Multiple Asymmetric Dimensions}

In this subsection, we explore the scenario where players differ along multiple dimensions simultaneously, such as attack success probabilities (\(p_1, p_2\)) and attack costs (\(c_1, c_2\)). This scenario captures a more realistic situation where participants have asymmetric strengths and weaknesses across multiple strategic factors. For example, players in real-world competitive environments often differ in both their ability to launch successful attacks and the costs they incur for such actions. For instance, a company with better technology might have a higher success probability but also face higher operational costs. These trade-offs affect the cooperation dynamics and require more nuanced strategies to sustain cooperative outcomes.

The payoffs for both players are adjusted to reflect asymmetries in both attack success probabilities and attack costs.

\begin{itemize}
    \item \textbf{Cooperation Payoff:}
    \[
    R_1 = R_2 = \frac{1}{2}.
    \]

    \item \textbf{Temptation Payoffs:}
    \begin{align*}
    T_1 &= p_1 + (1 - p_1) \left( \frac{1}{2} \right) - c_1, \\
    T_2 &= p_2 + (1 - p_2) \left( \frac{1}{2} \right) - c_2.
    \end{align*}

    \item \textbf{Sucker Payoffs:}
    \begin{align*}
    S_1 &= (1 - p_2) \left( \frac{1}{2} \right), \\
    S_2 &= (1 - p_1) \left( \frac{1}{2} \right).
    \end{align*}

    \item \textbf{Mutual Defection Payoff:}
    \begin{align*}
    Q_1 &= p_1 p_2 \frac{1}{2}\beta + p_1 (1 - p_2) + (1 - p_1)(1 - p_2) \left( \frac{1}{2} \right) - c_1, \\
    Q_2 &= p_1 p_2 \frac{1}{2}\beta + p_2 (1 - p_1) + (1 - p_1)(1 - p_2) \left( \frac{1}{2} \right) - c_2.
    \end{align*}
\end{itemize}

\paragraph{Cooperation Conditions with Multiple Asymmetries}

The cooperation conditions must now account for both the differences in attack success probabilities and attack costs. For cooperation to be sustainable, the following inequalities must hold under the Grim Trigger and Tit-for-Tat strategies.

\begin{theorem}
\label{thm:multiple_asym}
Let two players differ in both attack success probabilities and attack costs. The cooperation is sustainable if and only if:
\begin{align*}
\delta_1 \geq \frac{p_1 - 2c_1}{(1 - \beta)p_1 p_2 + p_2} \quad \text{and} \quad \delta_2 \geq \frac{p_2 - 2c_2}{(1 - \beta)p_1 p_2 + p_1} \tag{Grim Trigger}
\end{align*}
\begin{align*}
\delta_1 \geq \frac{p_1 - 2c_1}{p_2} \quad \text{and} \quad \delta_2 \geq \frac{p_2 - 2c_2}{p_1} \tag{Tit-For-Tat}
\end{align*}
\end{theorem}

\begin{proof}
The proof follows from the analysis of each player's incentive to defect based on their unique attack success probability and attack cost.
\end{proof}

\begin{corollary}
When players differ in both attack costs and success probabilities, the cooperation condition is determined by the player with the greater temptation to defect, which is jointly influenced by their attack cost and success probability.
\end{corollary}

\begin{proof}
If \(p_1 < p_2\) but \(c_1 < c_2\), the interaction between success probability and cost will determine which player has a stricter cooperation condition. The trade-off between higher probability and higher cost can either encourage or discourage defection, depending on the specific values.
\end{proof}

The framework we developed allows us to analyze more complex scenarios where players differ along multiple dimensions simultaneously. For instance, if Player 1 has a lower attack success probability (\(p_1 < p_2\)) but incurs a lower attack cost (\(c_1 < c_2\)), the outcome of cooperation will depend on the joint effect of these asymmetries. Specifically, Player 2’s higher success probability provides them with more opportunities to gain from defection, while Player 1’s lower attack cost reduces their barrier to engage in repeated attacks. Together, these factors shape the dynamics of cooperation.

These findings underscore the importance of identifying the player whose characteristics pose the greatest threat to long-term cooperation. Effective cooperative outcomes can be fostered through targeted incentives tailored to the asymmetries. For instance, Player 2 (the more capable but costlier player) could be incentivized to cooperate by sharing a portion of Player 1’s benefits or by imposing penalties that increase Player 2’s effective attack cost. Conversely, Player 1 could be compensated with a share of Player 2’s potential gains to discourage opportunistic behavior stemming from their lower-cost advantage.

Our framework systematically analyzes the interplay between multiple asymmetries and their impact on cooperation, enabling decision-makers to identify the key player whose characteristics pose the greatest threat to sustained cooperation. By modeling differences in attack success probabilities, attack costs, and other dimensions, it provides actionable insights into designing multi-faceted interventions—such as profit-sharing mechanisms or cost-based penalties—to address specific asymmetries. These contributions align with real-world practices, equipping decision-makers with tailored strategies to foster cooperation and mitigate risks associated with strategic imbalances.

% \subsection{Conclusion}

% The analysis of asymmetric player scenarios highlights that sustaining cooperation in competitive environments with unequal players is challenging. The player with the greatest incentive to defect critically influences the dynamics. Effective strategies to maintain cooperation must consider these asymmetries and target interventions accordingly.

\section{Multi-Player Scenarios}

This section derives the conditions under which cooperation can be sustained in scenarios with \(N\) players, of which \(M\) players defect, using different trigger strategies. The analysis extends the payoff structure and the grim trigger and tit-for-tat strategies to multi-player contexts, providing a generalized framework.

In a system with \(N\) players, the payoff dynamics are determined by the total number of defectors \(M\), where \(M < N\). The market is shared among successful attackers, and the outcome varies based on the proportion of successful attacks. %Notably, degradation occurs only if all \(N\) players attack successfully, resulting in a degraded payoff of \(\frac{1}{2}\beta < \frac{1}{N}\) shared among all players. Otherwise, normal market sharing rules apply.

The payoffs are defined as follows:
\begin{itemize}
    \item \(R\) represents the payoff for mutual cooperation, where all \(N\) players cooperate and share the normalized market value equally, giving \(R = \frac{1}{N}\).
    \item \(T\) is the payoff for the defectors who launch attacks while others cooperate. The expected payoff for a one-time defector includes the probabilities of successful attacks. If exactly \(k\) out of \(M\) defectors succeed (probability \( \binom{M}{k} p^k (1 - p)^{M - k} \)), the \(k\) successful attackers share the market equally. If no attack succeeds (probability \( (1 - p)^M \)), the market is shared equally among all \(N\) players.
  \[
  T = \sum_{k=1}^{M} \binom{M}{k} p^k (1 - p)^{M - k} \cdot \frac{1}{k} + (1 - p)^M \cdot \frac{1}{N} - c
  \]
  \item \(S\) is the payoff for a cooperator when \(M\) players defect. If all \(M\) attackers fail (probability \((1 - p)^M\)), the cooperator receives their share of the market, which is \( \frac{1}{N} \). If at least one attack succeeds (probability \(1 - (1 - p)^M\)), the cooperator receives nothing, as the successful attackers monopolize the market. The resulting payoff is:
  \[
  S = (1 - p)^M \cdot \frac{1}{N}
  \]
  \item \(Q\) captures the payoff for mutual defection, where all \(N\) players defect. If exactly \(k\) attacks succeed (probability \( \binom{N}{k} p^k (1 - p)^{N - k} \)), the \(k\) successful attackers share the market equally. If all \(N\) attacks succeed (probability \(p^N\)), the market degrades to \( \beta \), shared equally among the \(N\) players. If no attack succeeds (probability \( (1 - p)^N \)), the market is shared equally among all \(N\) players.
  \[
  Q = \sum_{k=1}^{N-1} \binom{N}{k} p^k (1 - p)^{N - k} \cdot \frac{1}{k} + p^N \cdot \frac{\beta}{N} + (1 - p)^N \cdot \frac{1}{N} - c
  \]
\end{itemize}

Now let’s plug in our formulas into the critical conditions for the Grim Trigger and Tit-for-Tat strategies, given:

Under Grim Trigger, the cooperation is sustainable if and only if:  
  \[
  \delta \geq \frac{T - R}{T - Q}.
  \]

Under Tit-for-Tat, the cooperation is sustainable if and only if:  
  \[
  \delta \geq \frac{T - R}{R - S}.
  \]

Let’s substitute the derived formulas step-by-step to better understand the conditions for sustaining cooperation.

\begin{theorem}[Condition to Sustain Cooperation in Multi-Player Setting]
In the Multi-Player setting with the Grim Trigger strategy, cooperation is sustainable if and only if:
\begin{align*}
    \delta &\geq \frac{T - R}{T - Q} \overset{\underset{\mathrm{def}}{}}{=} \delta^*_{\text{Multi,GT}} \tag{Grim Trigger} \\
    \delta &\geq \frac{T - R}{R - S}\overset{\underset{\mathrm{def}}{}}{=} \delta^*_{\text{Multi, TFT}} \tag{Tit-For-Tat}
\end{align*}
, where
\begin{align*}
&\delta^*_{\text{Multi,GT}} = 
\mathsmaller{
\frac{
\left( \sum_{k=1}^{M} \binom{M}{k} p^k (1 - p)^{M - k} \cdot \frac{1}{k} + (1 - p)^M \cdot \frac{1}{N} - c \right) - \frac{1}{N}
}{
\left( \sum_{k=1}^{M} \binom{M}{k} p^k (1 - p)^{M - k} \cdot \frac{1}{k} + (1 - p)^M \cdot \frac{1}{N} - c \right) 
- \left( \sum_{k=1}^{N-1} \binom{N}{k} p^k (1 - p)^{N - k} \cdot \frac{1}{k} + p^N \cdot \frac{\beta}{N} + (1 - p)^N \cdot \frac{1}{N} - c \right)
},
} \\
&\delta^*_{\text{Multi, TFT}} = \mathsmaller{\frac{\left( \sum_{k=1}^{M} \binom{M}{k} p^k (1 - p)^{M - k} \cdot \frac{1}{k} + (1 - p)^M \cdot \frac{1}{N} - c \right) - \frac{1}{N}}{\frac{1}{N} - (1 - p)^M \cdot \frac{1}{N}}}
\end{align*}
\end{theorem}

% \[
% \begin{aligned}
% c \geq & \ (1 - \delta) \left( \sum_{k=1}^{M} \binom{M}{k} \frac{p^k (1 - p)^{M - k}}{k} + \frac{(1 - p)^M}{N} \right) \\
% & + \delta \left( \sum_{k=1}^{N-1} \binom{N}{k} \frac{p^k (1 - p)^{N - k}}{k} + p^N \cdot \frac{\frac{1}{2}\beta}{N} + \frac{(1 - p)^N}{N} \right) - \frac{1}{N}
% \end{aligned}
% \]

\begin{corollary}  
Given the total number of players \( N \), when there are a substantial number of attackers, an increase in the number of attackers \( M \) widens the range of discount factors \(\delta\) that sustain cooperation.
\end{corollary}

As the number of attackers grows, the critical discount factor \(\delta^*_{\text{Multi}}\) decreases, making it easier to satisfy the cooperation condition.  The corollary highlights an important relationship between the number of attackers and the feasibility of sustaining cooperation. As the number of attackers \( M \) increases, the opportunity for any individual attacker to gain significant benefits through defection diminishes because the total payoff is spread across more players. In other words, the marginal gain from defection becomes smaller when more players are involved in attacks. Simultaneously, the punishment remains severe, which strengthens the incentive to cooperate. Consequently, even players who place less weight on future rewards (lower \(\delta\)) will still find cooperation preferable, widening the range of discount factors under which cooperation can be sustained.

\begin{proof}  
For large \( M \), the term \( T \) can be approximated by focusing on the dominant portion of its summation. The probability mass of the binomial distribution is centered around \( k \approx M p \), and since \(\frac{1}{k} \approx \frac{1}{M p}\) in that region, the sum \(\sum_{k=1}^{M} \binom{M}{k} p^k(1-p)^{M-k}\frac{1}{k}\) approaches \(\frac{1}{M p}\). Additionally, \((1 - p)^M \frac{1}{N}\) becomes negligible as \( M \) grows. Hence:
\[
T \approx \frac{1}{M p} - c.
\]
Substituting \( T \) into \(\delta^*_{\text{Multi,GT}}(M)\):
\[
\delta^*_{\text{Multi,GT}}(M) \approx \frac{\left(\frac{1}{M p}- c - \frac{1}{N}\right)}{\left(\frac{1}{M p} - c - Q\right)}.
\]
Define:
\[
x = \frac{1}{M p}, \quad A = c + \frac{1}{N}, \quad B = c + Q.
\]
Then:
\[
\delta^*_{\text{Multi,GT}}(M) \approx \frac{x - A}{x - B}.
\]
We also know that \( Q < \frac{1}{N} \), which implies:
\[
B = c + Q < c + \frac{1}{N} = A.
\]
Thus, \( A - B > 0 \).

To determine how \(\delta^*_{\text{Multi,GT}}(M)\) evolves as \( M \) grows, examine how it changes as \( x \) varies. Since \( x = \frac{1}{M p} \) decreases as \( M \) increases, the sign of the derivative with respect to \( x \) reveals the direction of change with respect to \( M \). The derivative with respect to \( x \) is:
\[
\frac{d}{dx}\left(\frac{x - A}{x - B}\right) = \frac{A - B}{(x - B)^2}.
\]

Since \( A - B > 0 \), the fraction \((x - A)/(x - B)\) is strictly increasing in \( x \). As \( M \) grows, \( x \) becomes smaller, so \(\delta^*_{\text{Multi,GT}}(M)\) decreases. In other words, increasing \( M \) reduces \(\delta^*_{\text{Multi,GT}}(M)\).

In conclusion, as \( M \) becomes large, the value of \(\delta^*_{\text{Multi,GT}}(M)\) is guaranteed to decrease. This analysis provides a clear asymptotic trend for \(\delta^*_{\text{Multi,GT}}(M)\), indicating that growth in \( M \) shifts the ratio toward smaller values.

% \[
% c \geq \sum_{k=1}^{M} \binom{M}{k} \frac{p^k (1 - p)^{M - k}}{k} + \frac{(1 - p)^M - 1}{N} - \delta \cdot \frac{1 - (1 - p)^M}{N}
% \]

Similarly, for the Tit-for-Tat strategy, we simplify the denominator: 
\[
R - S = \frac{1}{N} - \frac{(1 - p)^M}{N} = \frac{1 - (1 - p)^M}{N}.
\]
For large \( M \), \((1 - p)^M \rightarrow 0\) whenever \( 0 < p < 1 \). Thus:
\[
R - S \approx \frac{1}{N}.
\]
Next, examine the numerator for large \( M \). As before, the dominant term in 
\(\sum_{k=1}^{M} \binom{M}{k} p^k (1 - p)^{M-k}\frac{1}{k}\) is approximately \(\frac{1}{M p}\). Also, \((1 - p)^M \frac{1}{N}\) becomes negligible. Hence:
\[
T - R \approx \frac{1}{M p} - c - \frac{1}{N}.
\]
Putting these approximations together:
\[
\delta^*_{\text{Multi, TFT}} \approx \frac{\frac{1}{M p} - c - \frac{1}{N}}{\frac{1}{N}} = N\left(\frac{1}{M p} - c - \frac{1}{N}\right) =  \frac{N}{M p} - cN - 1.
\]
Since the only term in \(\delta^*_{\text{Multi, TFT}}\) that depends on \( M \) is \(\frac{N}{M p}\), which decreases as \( M \) increases, \(\delta^*_{\text{Multi, TFT}}(M)\) also decreases with increasing \( M \).  
\end{proof}  

\section{Conclusion}

This paper presents a game-theoretic analysis of ranking manipulation attacks in LLM-based search engines, where content providers strategically decide whether to engage in manipulative practices to gain competitive advantages. By modeling these interactions as an Infinitely Repeated Prisoners' Dilemma, we capture the unique characteristics of LLM-based systems, including stochastic attack success rates, attack costs, future discount rates, and market degradation effects. Our analysis reveals several key insights that have direct implications for platform designs, industry practices, and regulatory policies.

First, we find that cooperation sustainability depends critically on the interplay between immediate costs and long-term benefits. Content providers with high attack costs or strong forward-looking tendencies are more likely to maintain cooperative behavior, as the immediate gains from manipulation are outweighed by long-term market benefits. In practice, this suggests that platform operators should implement both immediate and long-term deterrence mechanisms. For example, search engines could impose escalating computational costs for realizing ranking manipulations, while simultaneously developing reputation systems that reward long-term cooperative behavior.

Second, we discover that the relationship between attack success probability and cooperation sustainability is non-monotonic. Intermediate success rates can sometimes incentivize more manipulation attempts than high rates, as they provide an optimal balance between potential gains and the combined loss of costs and degradation risks. This presents a crucial challenge for platform operators, suggesting that partial defenses might paradoxically be futile. Defensive measures aimed at capping attack success rates fail to meaningfully reduce manipulation incentives. This finding fundamentally challenges traditional approaches to platform security. Rather than investing heavily in technical measures within these futile defense regions, platforms should redirect resources toward economic deterrence mechanisms and reputation systems that create long-term incentives for cooperation.

In examining asymmetric player scenarios, we find that system stability is primarily determined by the player with the strongest incentive to defect. This suggests that platforms should focus their monitoring and enforcement efforts on the most potential attackers. Platforms could implement tiered security measures that apply stricter scrutiny to larger content providers or those with a history of sophisticated ranking manipulation practices. Additionally, platforms might consider implementing compensatory mechanisms that help level the playing field between players with different capabilities, thereby reducing the incentive for more capable players to exploit their advantages.

The policy implications can be read directly through the model parameters. Defenses that make attacks harder to engineer, such as adversarial-content detection, randomized ranking audits, or stronger separation between retrieved data and model instructions, primarily reduce $p$. Measures that increase the effort needed to develop, test, and maintain attacks, such as rate limits, adaptive audits, and escalating review for repeated suspicious edits, increase $c$. Enforcement mechanisms that demote or remove pages involved in detected simultaneous manipulation lower the effective mutual-defection payoff $Q$, which in our reduced-form model corresponds to a lower effective $\beta$ from the attackers' perspective. Reputation systems, long-term contracts, and persistent audit records increase the value of future cooperative payoffs relative to short-run manipulation gains, which corresponds to a higher effective $\delta$. The non-monotonicity results show that reducing $p$ alone may not be enough; a defense is more reliable when it jointly lowers attainable $p$, raises $c$, and reduces the continuation payoff from mutual defection.

The broader applicability of our findings extends beyond LLM-based search engines to other LLM-driven platforms. Our framework for analyzing the dynamics of players on LLM-based search platforms could inform the design of recommendation systems, content moderation platforms, and more. For example, social media platforms could apply our insights to design more effective content ranking systems that are resistant to manipulation while maintaining fair competition among content creators. Additionally, the same principles may guide developers in refining automated moderation strategies, helping to address emerging content patterns and preserve user trust in the relevance and integrity of ranked outputs.

Several exciting and inspiring directions emerge for future research based on our findings. First, our model opens up opportunities to analyze two-sided markets, where both content providers and platform operators act strategically. Extending our framework to these scenarios could illuminate the dynamic interplay between manipulation strategies and defensive measures, offering a richer understanding of how both sides evolve over time. Second, our work inspires deeper exploration into user behavior and market feedback, which could uncover how manipulation influences long-term platform viability and user trust. By integrating these dynamics into the model, researchers can provide actionable insights into designing systems that preserve user satisfaction while deterring manipulative practices.
Another promising avenue is examining the role of information asymmetry, as real-world participants often operate with imperfect knowledge of others' capabilities and actions. Investigating how this asymmetry shapes strategic behavior could reveal new approaches for designing platforms that are resilient to uncertainty and exploitation.
%Our findings also highlight opportunities to study the effects of market structures and competitive dynamics on manipulation incentives. For instance, future work could explore how different market configurations, such as monopolistic versus competitive environments, influence the sustainability of cooperative behavior among content providers.
Finally, as LLMs become increasingly sophisticated, our research lays the groundwork for understanding how advancements in language model capabilities reshape the balance between manipulation and defense. This direction could inspire the development of adaptive and forward-looking platform strategies that align with the evolving technological landscape.

By connecting game-theoretic insights to practical platform design considerations, our work provides a foundation for developing more secure and sustainable LLM-driven information systems. This work contributes to both the theoretical understanding of LLM market mechanisms and its practical application in real-world security problems of LLM-driven systems, offering actionable recommendations for platform operators, policymakers, and industry stakeholders. Through these contributions, we aim to not only advance the academic understanding of manipulation dynamics in LLM-based systems but also equip practitioners with the tools to design robust and equitable platforms that foster trust, fairness, and long-term stability.

\bibliographystyle{iclr2024_conference}
\bibliography{references}

\end{document}